\documentclass[twoside,11pt]{article}

\usepackage{blindtext}

\usepackage[abbrvbib, preprint]{jmlr2e}

\usepackage{amsmath}

\usepackage{multirow}
\usepackage{tcolorbox}
\usepackage{threeparttable}
\usepackage{algorithm,algorithmic}

\newcommand{\A}{\mathcal{A}}

\newcommand{\mw}{\mathbf{w}}
\newcommand{\mv}{\mathbf{v}}

\newcommand{\ls}{\left\|}
\newcommand{\rs}{\right\|_2}
\newcommand{\lp}{\left(}
\newcommand{\rp}{\right)}

\newcommand{\lk}{\left[}
\newcommand{\rk}{\right]}

\newtheorem{assumption}{Assumption}[section]
\usepackage{cleveref}
\crefname{equation}{Eq.}{Eqs.}
\crefname{definition}{Definition}{Definitions}
\crefname{assumption}{Assumption}{Assumptions}
\crefname{theorem}{Theorem}{Theorems}
\crefname{remark}{Remark}{Remarks}
\crefname{lemma}{Lemma}{Lemmas}
\crefname{corollary}{Corollary}{Corollaries}
\crefname{proposition}{Proposition}{Propositions}
\crefname{section}{Section}{Sections}
\crefname{subsection}{Subsection}{Subsections}
\crefname{example}{Example}{Examples}
\crefname{table}{Table}{Tables}
\crefname{problem}{Problem}{Problems}
\crefname{algorithm}{Algorithm}{Algorithms}
\crefname{figure}{Figure}{Figures}
\crefname{property}{Property}{Properties}
\crefname{appendix}{Appendix}{Appendices}

\usepackage{lastpage}
\jmlrheading{23}{2026}{1-\pageref{LastPage}}{1/21; Revised 5/22}{9/22}{21-0000}{Wang, Luo, Wen, Li and Sun}

\ShortHeadings{High-Probability Generalization of D-SGD}{Wang, Luo, Wen, Li and Sun}
\firstpageno{1}

\begin{document}

\title{Unveiling High-Probability Generalization in Decentralized SGD}

\author{\name Jiahuan Wang \email wangjiahuan@nudt.edu.cn \\
       \name Ping Luo \email luoping@nudt.edu.cn \\
       \name Ziqing Wen \email zqwen@nudt.edu.cn \\
       \name Dongsheng Li \email lds1201@163.com \\
       \name Tao Sun\thanks{Corresponding Author.}  \email nudtsuntao@163.com \\
       \addr College of Computer Science and Technology\\
       National University of Defense Technology\\
       Changsha, Hunan 410073, China
       }

\editor{}

\maketitle

\begin{abstract}
	Decentralized stochastic gradient descent (D-SGD) is an efficient method for large-scale distributed learning. Existing generalization studies mainly address expected results, achieving rates limited to $\mathcal{O}\lp\frac{1}{\delta \sqrt{mn}}\rp$, where $\delta$ is the confidence parameter, $m$ the number of workers, and $n$ the sample size. When $m=1$, D-SGD reduces to traditional SGD, whose optimal high-probability generalization bound is $\mathcal{O}\left(\frac{1}{\sqrt{n}}\log (1/\delta)\right)$. This discrepancy reveals a gap between high-probability guarantees for SGD and those for D-SGD. To close this, we develop a high-probability learning theory for D-SGD, aiming for the optimal $\mathcal{O}\left(\frac{1}{\sqrt{mn}}\log (1/\delta)\right)$ rate. We refine bounds for D-SGD using pointwise uniform stability in distributed learning—a weaker notion than uniform stability—and analyze them across convex, strongly convex, and non-convex settings. We also provide high-probability results for gradient-based measures in non-convex cases where only local minima exist, and derive optimization error and excess risk bounds. Finally, accounting for communication overhead, we analyze generalization bounds for local models within time-varying frameworks.
\end{abstract}


\section{Introduction}
In recent years, the study of stability and generalization properties of randomized algorithms \citep*{bousquet2002stability,elisseeff2005stability,hardt2016train,kuzborskij2018data,lei2020fine,yuan2023exponential,yuan2024l_2,ramezani2024generalization} has gained significant traction, driven by their crucial role in machine learning community. These algorithms, particularly stochastic gradient descent (SGD) and its variants, have become indispensable tools in both theoretical research and practical applications, especially in large-scale settings \citep{li2014efficient,lin2018don,woodworth2020local,woodworth2020minibatch}. Among these, decentralized stochastic gradient descent has emerged as an efficient approach for distributed learning, enabling multiple agents to collaboratively learn a model without requiring central coordination \citep{nedic2009distributed,sundhar2010distributed}. Studies have shown that D-SGD not only reduces communication bottlenecks but can also achieve comparable or even superior performance to centralized methods in certain scenarios \citep{lian2017can}. This makes D-SGD particularly well-suited for large-scale optimization tasks in distributed systems \citep{shi2015extra,yuan2016convergence,lian2017can,koloskova2020unified,sun2021stability,yuan2023removing}.

Despite the widespread application of decentralized stochastic gradient descent, the theoretical understanding of its stability and generalization properties remains limited. A growing body of research has explored the generalization of D-SGD and its variants \citep{sun2021stability,zhu2022topology,deng2023stability,bars2023improved,zhu2024stability,wang2024towards}, but much of this work focuses on expected generalization bounds. These bounds typically characterize the algorithm's performance in an average-case scenario, providing valuable insights but often lacking robustness in settings where guarantees on tail behavior are critical.
For example, we can establish the uniform stability bound for D-SGD (More details can be found in \cref{section Convex Case})
\begin{align*}\label{uniform bound}
	\epsilon_{\textrm{uniform}}=4\beta L^2\sum_{t=1}^{T}\eta_t\sum_{q=1}^{t-1}\eta_q\lambda^{t-q-1}+\frac{2L^2}{mn}\sum_{t=1}^{T}\eta_t.
\end{align*}
Then we obtain the generalization error in expectation \citep{hardt2016train,sun2021stability} under $\eta_t=\eta$ with
\begin{align*}
	\left| \mathbb{E}_{S,\A}\lk R_S(\A(S))-R(\A(S))\rk\right|\leq 2L^2\lp\frac{2\eta^2\beta T}{1-\lambda}+\frac{\eta T}{mn}\rp.
\end{align*}
According to Markov's inequality, with $1-\delta$ probability, we have 
\begin{align*}
	R(\A(S))-R_S(\A(S))\leq\frac{2L^2}{\delta}\lp\frac{2\eta^2\beta T}{1-\lambda}+\frac{\eta T}{mn}\rp.
\end{align*}
If $\eta=\frac{1}{\sqrt{T}}$ and $T=mn$, we can only get $R(\A(S))-R_S(\A(S))=\mathcal{O}\lp\frac{1}{\delta \sqrt{mn}}\rp.$

In contrast, high-probability generalization bounds provide stronger guarantees by ensuring desirable performance with high certainty, even under worst-case conditions. While such bounds have been extensively studied in centralized settings \citep{bousquet2002stability,feldman2018generalization,feldman2019high,bousquet2020sharper,klochkov2021stability}, their application to D-SGD remains conspicuously underexplored. Addressing this gap is crucial to building a more complete theoretical framework for D-SGD, offering guarantees that are not only rigorous but also practically reliable. Motivated by this gap, we raise the following questions:
\begin{tcolorbox}[colback=pink!10!white, colframe=pink!60!black]
	\begin{center}
		\textit{{What are the high-probability results of the stability and generalization analysis of D-SGD? Can D-SGD achive the sharp high-probability convergence rate of $\mathcal{O}(\frac{1}{\sqrt{mn}})$?}}
	\end{center}
\end{tcolorbox}
This paper addresses the aforementioned questions by bridging key theoretical gaps and deriving satisfactory generalization bounds. 
\begin{table*}
	\begin{center}
		\renewcommand\arraystretch{1.4}
		\begin{threeparttable}
			\begin{tabular}{cccc}
				\hline
				Update Type&Reference&Situation&Generalization Bounds\\
				\hline
				{I \eqref{I}}&\cite{richards2020graph}&$\eta\leq{2}/{\beta}$&$\mathcal{O}\lp\frac{1}{\delta \sqrt{mn}}\rp$\\
				\hline
				\multirow{3}{*}{II \eqref{II}}&\cite{sun2021stability}&$\eta\leq{2}/{\beta}$&$\mathcal{O}\lp\frac{1}{\delta \sqrt{mn}}\rp$\\
				&\cite{bars2023improved}&$\eta\leq{2\min_{k}P_{kk}}/{\beta}$&$\mathcal{O}\lp\frac{1}{\delta \sqrt{mn}}\rp$\\
				&Ours&$\eta\leq{2}/{\beta}$&$\mathcal{O}\lp\frac{1}{\sqrt{mn}}\log (1/\delta)\rp$\\
				\hline
			\end{tabular}
		\end{threeparttable}
		\caption{Summary of stability and generalization analysis for D-SGD under convex case. ($P_{kk}$ represents the $k$-th diagonal element of the gossip matrix; $\eta$ means stepsize; $\beta$ denotes smoothness property.)}
		\label{Compare 1}
	\end{center}
\end{table*}
\subsection{Difference and Technical Challenges}
The main difference between our work and previous studies lies in the type of stability used for deriving generalization bounds. While previous approaches typically rely on uniform stability \citep{richards2020graph,sun2021stability}, we introduce a weaker but more general notion of pointwise uniform stability for distributed learning. By leveraging this weaker stability concept, we are able to refine the stability-based generalization results and obtain optimal convergence rates for D-SGD, specifically achieving $\mathcal{O}(\frac{1}{\sqrt{mn}})$ convergence.

One of the primary technical challenges we address is the derivation of high-probability generalization bounds for D-SGD. Unlike expected bounds \citep{richards2020graph,sun2021stability,zhu2022topology,bars2023improved}, high-probability results require careful treatment of tail behavior in the random variables involved, which adds significant complexity to the analysis. In non-convex settings, deriving high-probability optimization error bounds is another technical hurdle. We tackle this using martingale difference sequence techniques, which allow us to analyze the stochastic behavior of local updates in a rigorous manner. Additionally, we address the challenge of analyzing the generalization behavior of local models within the time-varying D-SGD framework. In decentralized settings, communication overhead and the impact of changing network topologies \citep{nedic2017achieving,koloskova2020unified,le2023refined} present unique difficulties. Our analysis of local models provides a clearer understanding of how these factors influence the generalization performance of D-SGD. More details can be found in Section \ref{challenges}.
\subsection{Contributions}
To the best of our knowledge, this work makes a novel attempt to examine the high-probability generalization behavior of D-SGD, addressing both generalization error and optimization error. We provide comprehensive theoretical results for D-SGD under convex, strongly convex, and non-convex settings. Our main contributions are summarized as follows:
\begin{itemize}
	\item We provide refined stability and generalization bounds for D-SGD under pointwise uniform stability for distributed learning (See \cref{pointwise uniform dl}), which is weaker than uniform stability. In particular, our results are derived with high probability, offering stronger guarantees compared to expected bounds. The high-probability bound achieves the sharp convergence rate of $\mathcal{O}(\frac{1}{\sqrt{mn}})$ , which aligns with the optimal rates established in centralized settings. These bounds are applicable to a broad range of settings, including convex (See \cref{gen convex case}), strongly convex (See \cref{strong convex theorem 1}), and non-convex scenarios (See \cref{nonconvex case theorem 1}). 
	\item We present high-probability bounds for optimization error and excess risk in non-convex problem under gradient-dominance condition (See \cref{nonconvex optim}), providing a more robust understanding of D-SGD's performance in these challenging settings. 
	\item We derive generalization error bounds for the local models in time-varying D-SGD (See \cref{convex local,strong convex local,non convex local}), which are especially relevant in decentralized scenarios where communication costs or instability could be a challenge. The derived bounds illustrate how the changing topology impacts the local models.
\end{itemize}

\section{Related Work}

\subsection{Stability and Generalization of D-SGD and its variants}
Algorithmic stability is a central concept in statistical learning theory, serving as a key measure of how sensitive an algorithm is to small changes in the training data \citep{bousquet2002stability,elisseeff2005stability,shalev2010learnability}. It provides dimension-independent generalization bounds, making it an essential tool in understanding the performance of learning algorithms across various settings \citep{hardt2016train,bousquet2020sharper,lei2020fine}.
\cite{richards2020graph} first introduced a generalization bound for D-SGD using the concepts of algorithmic stability and Rademacher complexity, applicable in both smooth and nonsmooth settings. \cite{sun2021stability} proposed stability and generalization bounds for D-SGD based on uniform stability. \cite{zhu2022topology} examined the impact of communication topology on D-SGD, employing the on-average stability framework under the H\"{o}lder-smooth condition \citep{lei2020fine}. Additionally, \cite{bars2023improved} established similar generalization bounds for D-SGD as those for classical C-SGD \citep{hardt2016train}. Recently, \cite{zeng2025stability} provided sharper bounds for D-SGD without requiring Lipschitz or smoothness.The specific differences in the approaches to D-SGD in the aforementioned works will be clarified in detail in \cref{Compare 1}.

More recently, several studies have explored the generalization properties of various D-SGD variants, including asynchronous D-SGD \citep{deng2023stability}, D-SGD with minibatch \citep{wang2024towards}, D-SGDA \citep{zhu2024stability}, and zeroth-order optimization \citep{wang2024towards,hu2025stability}. However, it is important to note that all of these works focus on expected generalization results, which are based on the average behavior of the algorithm across random perturbations of the training data. These approaches do not consider high-probability generalization bounds, which are central to this paper's investigation. Specifically, we aim to extend the analysis of D-SGD stability to high-probability generalization bounds, exploring whether the existing stability-based techniques can be applied to obtain tighter and more robust bounds in this context.

\subsection{High-probability bounds for uniform stable algorithms}
We also provide a brief review of the development of high-probability generalization bounds for uniformly stable algorithms. The foundational result was derived by \cite{bousquet2002stability}, which states that the generalization error bound for uniformly stable algorithms is of the form $\mathcal{O}\lp\epsilon n^{1/2}+M{n}^{-1/2}{\log^{1/2} (1/\delta)}\rp$, where $\epsilon$ denotes the stability parameter, $M$ is the upper bound of the loss function, $n$ is the sample size, and $\delta$ is the confidence parameter. However, in the regime where $\epsilon\leq n^{-1/2}$, the bound becomes trivial and vanishes. Subsequently, \cite{feldman2018generalization,feldman2019high} presented a stronger generalization bound, which is of the form
$\mathcal{O}\lp\epsilon\log n\lp\log n+\log (1/\delta)\rp+M{n}^{-1/2}{\log^{1/2} (1/\delta)}\rp$.
This result introduces the $\log^2 n$ term, but it also makes the bound less tight. Building on this, \cite{bousquet2020sharper} derived a more refined exponential bound, which improves the generalization error to $\mathcal{O}\lp\epsilon\log n\log (1/\delta)+M{n}^{-1/2}{\log^{1/2} (1/\delta)}\rp$.
In contrast to these results, \cite{klochkov2021stability} considered the excess risk bound under the Bernstein condition, deriving a generalization bound of the form $\mathcal{O}(1/n)$. Recently, \cite{fan2024high} considered a weaker form of pointwise uniform stability and derived similarly tight high-probability generalization bounds, which are also applicable in more general settings with relaxed stability assumptions.
\section{Preliminaries}
Consider a distributed system with $m$ workers, and let
$S=\{S_1,S_2,\cdots,S_m\}=\{Z_{1(1)},\cdots,Z_{k(r)},\cdots,Z_{n(m)}\}$
represent the training dataset, which each $Z_{k(r)}$ is independently and identically drawn from a probability measure $\mathcal{D}$.
We focus on finding a model $\mw$ from a given parameter space $\mathcal{W}$ that minimizes its population risk (or expected risk), which is defined as follows:
$
\min_{w\in\mathcal{W}} R(\mw):=\mathbb{E}_{Z\sim \mathcal{D}}\lk f\lp\mw;Z\rp\rk,
$
where $f\lp\mw;Z\rp$ represents a loss function which measures the performance of a model $\mw$ on data $Z$.
Since $\mathcal{D}$ is unknowable and unmeasurable, we cannot directly obtain the optimal solution for $R(\mw)$. Therefore, we consider its alternative: Empirical Risk, defined as
$
R_{S}(\mw)=\frac{1}{mn}\sum_{r=1}^{m}\sum_{k=1}^{n}f\lp\mw;Z_{k(r)}\rp.
$
And we assume that the empirical risk of $r$-th local model is $R_{S_{r}}=\frac{1}{n}\sum_{k=1}^{n}f\lp\mw;Z_{k(r)}\rp$. Let $\mathcal{A}(S)$ denote the model obtained by applying algorithm $\mathcal{A}$ (such as SGD or D-SGD) to the dataset $S$. Although $\mathcal{A}(S)$ may demonstrate a low empirical risk during training by perfectly fitting the examples, this empirical success does not guarantee a correspondingly low population risk. Therefore, it is natural to investigate the difference between population risk and empirical risk
$
	R(\mathcal{A}(S))-R_{S}(\mathcal{A}(S)).
$
We also focus on the excess generalization error $R(\A(S))-R(\mw^*)$, where $\mw^*$ is the theoretically minimizer of $R(\mw)$. This term can be structured as follows:
\begin{align*}
	&R(\A(S))-R(\mw^*)\\
	=&\underbrace{R(\A(S))-R_S(\A(S))}_{\textrm{Generalization\quad Error}}+\underbrace{\ R_S(\A(S))-R_S({\mw}_{R}^{*})}_{\textrm{Optimization\quad Error}}
	+\underbrace{R_S({\mw_{R}^{*}})-R_S(\mw^*)}_{\leq 0}+\underbrace{R_S({\mw^{*}})-R(\mw^*)}_{\textrm{Test \quad Error}},
\end{align*}
where $\mw_{R}^{*}$ denotes the minimizer of $R_S(\mw)$. Due to $R_S(\mw_{R}^{*})\leq R_S(\mw^*)$, the third item is always less than or equal to 0. For the test error term, Bernstein’s inequality \citep{hoeffding1994probability} can be applied to bound it. Thus, the analysis can be simplified to focus on the first two terms: the generalization error and the optimization error.

Before introducing the pointwise uniform stability \citep{fan2024high}, which is the primary focus of this paper, we first present uniform stability, a widely recognized tool for analyzing generalization error \citep{bousquet2002stability, elisseeff2005stability, agarwal2009generalization, hardt2016train,bousquet2020sharper,richards2020graph}.
\begin{definition}[Uniform Stability]
	A stochastic algorithm $\A$ is considered $\epsilon$-uniformly stable if, for any two training datasets $S$ and $\tilde{S}$ that differ by at most one example, the following holds:
	$$
	\sup_{Z}\left| f\lp\A(S);Z\rp-f\lp\A(\tilde{S});Z\rp\right|\leq\epsilon_{\emph{uniform}}.
	$$
\end{definition}
This means that the difference in loss between the outputs of the algorithm on the two datasets is bounded by $\epsilon_{\textrm{uniform}}$.
\begin{definition}[Pointwise Uniform Stability for Distributed Learning]\label{pointwise uniform dl}
	Let $\A$ be the random algorithm and suppose that $\epsilon=\lp\epsilon_{11},\cdots,\epsilon_{1n},\cdots,\epsilon_{m1},\cdots,\epsilon_{mn}\rp$, where $\epsilon_{rk}>0$. 
	
	\textit{\textbf{Function Value Stability:}} We say $\A$ is $\epsilon$-pointwise uniform stability for distributed learning in function values if  
	\begin{equation}
		\sup_{Z}\left|f\lp\A(S);Z\rp-f\lp\A(S^{(rk)});Z\rp\right|\leq \epsilon_{rk},
	\end{equation}
	where $S^{(rk)}$ denotes the dataset obtained by replacing the $k$-th sample in the \(r\)-th subset of $S$.  
	
	\textit{\textbf{Gradient Stability: }}We say $\A$ is $\epsilon$-pointwise uniform stability in gradients for distributed learning  if  
	\begin{equation*}
		\sup_{Z}\left\|\nabla f\lp\A(S);Z\rp-\nabla f\lp\A(S^{(rk)});Z\rp\right\|\leq \epsilon_{rk}.
	\end{equation*}
\end{definition}
For convenience, we abbreviate the following terms
\begin{equation*}
	\varDelta_{rk}^2=\frac{1}{mn}\sum_{r=1}^{m}\sum_{k=1}^{n}\epsilon_{rk}^2,\quad \varDelta_{rk}=\frac{1}{mn}\sum_{r=1}^{m}\sum_{k=1}^{n}\epsilon_{rk}.
\end{equation*}
The following lemma characterizes the relationship between pointwise uniform stability and the generalization error in the high-probability sense. 
\begin{lemma}[\citep{fan2024high}, Theorem 2 and 3]\label{Lemma 1}
	Let $\epsilon=\lp\epsilon_{11},\cdots,\epsilon_{1n},\cdots,\epsilon_{m1},\cdots,\epsilon_{mn}\rp$, where $\epsilon_{ij}>0$ for all $i,j$. Consider a distributed learning algorithm $\A$, and suppose that the loss function $f(\A(S);Z)$ is bounded by a constant $M$ for any dataset $S$ and data $Z$. If $\A$ is $\epsilon$-pointwise uniform stability for distributed learning in function values, then for $\delta\in(0,1)$, with probability at least $1-\delta$, the following inequality holds: 
	\begin{align*}
		\big|R(\A(S))-R_S(\A(S))\big|
		\lesssim\frac{M\log^{\frac{1}{2}}(1/\delta)}{\sqrt{mn}}+\lk\varDelta_{rk}^2\rk^{\frac{1}{2}}\log (mn)\log (1/\delta),
	\end{align*}
	where
	$
		\lk\varDelta_{rk}^2\rk^{\frac{1}{2}}=\lk\frac{1}{mn}\sum_{r=1}^{m}\sum_{k=1}^{n}\epsilon_{rk}^2\rk^{\frac{1}{2}}.
	$
	Moreover, if the gradients satisfy $\ls\nabla f(\A(S);Z)\rs\leq M$, we can also get
	\begin{align*}
		\ls\nabla R(\A(S))-\nabla R_S(\A(S))\rs
		\lesssim\frac{M\log^{\frac{1}{2}}(1/\delta)}{\sqrt{mn}}+\lk\varDelta_{rk}^2\rk^{\frac{1}{2}}\log (mn)\log (1/\delta).
	\end{align*}
\end{lemma}
\begin{remark}
	It is easy to observe that the above result differs slightly from the one in \cite{fan2024high}. We extend \cite{fan2024high}
	's result to the distributed learning scenario, taking into account the number of machines. The result we aim to compare with is the $\mathcal{O}(\frac{1}{\sqrt{n}})$-optimal result from \cite{bousquet2020sharper}. The main difference between the two lies in the type of stability used; however, their convergence rates are identical. If we can demonstrate that pointwise uniform stability leads to smaller results compared to uniform stability, it naturally follows that the corresponding high-probability generalization bounds will also be tighter \cite{fan2024high}.
\end{remark}
\subsection{Decentralized SGD}
D-SGD was proposed in \cite{lian2017can} and can be broken down into the following steps:

$\bullet$ \textbf{Gradient Calculation:} Each $i$-th node computes the local gradient $\nabla f(\mathbf{w}^{t}(i);Z_{j_{t}(i)})$, where $\mw^{t}(i)$ denotes the $i$-th local model parameter, $Z_{j_{t}(i)}$ is the sample drawn uniformly from $S_i$ at the $t$-th iteration.

$\bullet$ \textbf{Gossip Communication:} Each $i$-th node shares its model parameters $\mw^{t}(i)$ with neighboring nodes according to a decentralized communication graph $P$, where $P=[P_{ij}]\in\mathcal{R}^{m\times m}$ represents the communication link between each nodes. More details about the gossip matrix $P$ can see in the \cref{gossip}. Notably, in the final section (See \cref{local model}), we also consider a time-varying gossip matrix $P^t$ \citep{koloskova2020unified,le2023refined}, which can be also interpreted as a form of delayed communication.

$\bullet$ \textbf{Local Model Update:} Each local node updates its model $\mw^{t}(i+1)$ by performing a gradient descent step based on the aggregated model parameters received from its neighbors during the gossip step. 

The detailed algorithm can be found in \cref{D-SGD}, which shows the flow of the D-SGD process.

\begin{algorithm}
	\caption{Decentralized Stochastic Gradient Descent (D-SGD)}
	\label{D-SGD}
	\begin{algorithmic}[1] 
		\REQUIRE Initialize $\forall i, \mathbf{w}^{1}(i)=\mathbf{w}^{1}$, stepsizes $\{\eta_t\}_{t=1}^{T}$,weight matrix $P$ and the iteration number $T$. 
		\FOR {$ t=1,2,\cdots,T$ }
		\FOR {$ i=1,2,\cdots,m$}
		\STATE {Sample $Z_{j_t(i)}$ uniformly from the local dataset $S_i$}
		\STATE{$\mathbf{w}^{t+\frac{1}{2}}(i)=\sum_{l=1}^{m}P_{il}\mathbf{w}^{t}(l)$  \hfill {$\triangleright$ gossip communication}}
		\STATE {$\mathbf{w}^{t+{1}}(i)=\mathbf{w}^{t+\frac{1}{2}}(i)-{\eta_t}\nabla f(\mathbf{w}^{t}(i);Z_{j_{t}(i)})$ \hfill {$\triangleright$ local update}}
		\ENDFOR
		\ENDFOR
		\ENSURE $\mw^{T+1}=\frac{1}{m}\sum_{i=1}^{m}\mw^{T+1}(i)$
	\end{algorithmic} 
\end{algorithm}

\begin{definition}[Gossip matrix]
	\label{gossip}
	The gossip matrix $P$ satisfies the following properties:
	(1) $P$ is symmetric, i.e., $P=P^T$;
	(2) For any $i,j\in m$, $P_{ij}\in[0,1]$;
	(3) $\mathbf{1}_m^TP=P\mathbf{1}_m$, where $\mathbf{1}_m$ is the vector of ones.
\end{definition}
Let $\lambda_i$ be the $i$-th largest eigenvalue of $P$ \citep{sun2021stability,zhu2022topology,deng2023stability,bars2023improved,wang2024towards}, and define 
$
	\lambda:=\max\{|\lambda_2(P)|,|\lambda_m(P)|\}.
$
The gossip matrix is characterized by $0\leq\lambda<1$. For a connected graph, $\lambda=0$ implies a fully connected topology, where all entries of $P$ are $\frac{1}{m}$. In \cite{lian2017can}, it is observed that Steps 4 and 5 can be swapped without affecting the convergence rate. This flexibility allows for parallel computation of local gradients and communication, effectively hiding gossip time if it is shorter than computation time. The two update formulas are:
\begin{align}\label{I}
	\mathbf{w}^{t+{1}}(i)=&\sum_{l=1}^{m}P_{il}\lk\mathbf{w}^{t}(l)-\eta_t\nabla f(\mathbf{w}^{t}(l);Z_{j_{t}(l)})\rk,\\
	\mathbf{w}^{t+{1}}(i)=&\sum_{l=1}^{m}P_{il}\mathbf{w}^{t}(l)-\eta_t\nabla f(\mathbf{w}^{t}(i);Z_{j_{t}(i)}).\label{II}
\end{align}
However, swapping the two steps introduces stability analysis challenges. The second formula involves non-global updates, which may require stronger conditions to recover centralized results \citep{le2023refined} or reflect the decentralization impact \citep{sun2021stability,deng2023stability}.
\section{High-probability Bounds}\label{Assumption}
This section presents our main results on the generalization bounds of D-SGD algorithms based on pointwise uniform stability. Before starting the analysis, we need to introduce the following assumptions.
\begin{assumption}[$L$-Lipschitz]
	\label{Lipschitz}
	For any $Z\sim\mathcal{D}$ and $\mw,\tilde{\mw}\in\mathcal{W}$, $f(\mw;Z)$ is $L$-Lipschitz with respect to $\mw$ if 
	$
		|f(\mw;Z)-f(\tilde{\mw};Z)|_2 \leq L\|\mw-\tilde{\mw}\|_2.
	$
\end{assumption}
\begin{remark}
	This property ensures that the function does not change too rapidly, providing a bound on the rate of change. It implies that $\|\nabla f(\mw;Z)\|$ is bounded by $L$.
\end{remark}
\begin{assumption}[$\beta$-Smooth]
	\label{Smooth}
	For any $Z\sim\mathcal{D}$ and $\mw,\tilde{\mw}\in\mathcal{W}$, $f(\mw;Z)$ is $\beta$-Smooth with respect to $\mw$ if 
	$
		\|\nabla f(\mw;Z)-\nabla f(\tilde{\mw};Z)\|_2 \leq \beta\|\mw-\tilde{\mw}\|_2.
	$
\end{assumption}
\begin{remark}
	First-order smoothness means that its first derivative exists and is continuous across its domain.It also have the following property
	$
		f(\mw;Z)\leq f(\tilde{\mw};Z)+\langle \mw-\tilde{\mw},\nabla f(\tilde{\mw};Z)\rangle+\frac{\beta}{2}\ls \mw-\tilde{\mw}\rs^2.
	$
\end{remark}

\begin{assumption}[Bound Variance]\label{bound variance}
	\label{Bound Variance}
	We assume that for any $t\in \mathcal{N}$, there exists a constant $\sigma>0$,
	\begin{equation}
		\mathbb{E}_{\mathbf{j^t}}\left[\ls\frac{1}{m}\sum_{i=1}^{m}\nabla f\lp\mw^t ; Z_{j_t(i)}\rp-\nabla R_{S}\lp\mw^t\rp\rs^2\right]\leq\sigma^2,
	\end{equation}
	where $\mathbb{E}_{\mathbf{j^t}}=\mathbb{E}_{\{j_t(1),j_t(2,\cdots,j_t(m))\}}$ denotes the expectation with respect to $\mathbf{j^t}$.
\end{assumption}
\begin{remark}
	The above assumption plays a crucial role in controlling the fluctuations of gradient estimates in non-convex optimization. It helps ensure the stability and convergence of algorithms like stochastic gradient descent, even in the presence of non-convex objective functions \cite{ghadimi2013stochastic,kuzborskij2018data,lei2021generalization,lei2021learning}.
\end{remark}
\subsection{Convex Case}

\begin{theorem}\label{convex case theorem 1}
	Let $f(\mw;Z)$ be convex, L-Lipschitz (see \cref{Lipschitz}) and $\beta$-smooth (see \cref{Smooth}) for any Z with respect to $\mw$. If the stepsizse $\eta_t\leq2/\beta$, then D-SGD with $T$ iterations satisfies the pointwise uniform stability in fuction values with $\varDelta_{rk}=
	4\beta L^2\sum_{t=1}^{T}\eta_t\sum_{q=1}^{t}\eta_q\lambda^{t-q}+\frac{2L^2}{mn}\sum_{t=1}^{T}\eta_t$ and
	\begin{align*}
		\varDelta_{rk}^2=\frac{4L^4}{n}\sum_{k=1}^{n}\lp\sum_{t=1}^{T}2\beta \eta_t\sum_{q=1}^{t-1}\eta_q\lambda^{t-q-1}+\frac{1}{m}\eta_t\mathbb{I}_{[j_t=k]}\rp^2.
	\end{align*}
	Here, $\mathbb{I}_{[\cdot]}$ denotes the indicator function.
\end{theorem}
\begin{remark}
	Compared to the high-probability generalization bounds for (centralized) SGD in \cite{fan2024high}, our analysis considers the decentralized setting and also allows for dynamic step sizes. A notable aspect of our result is the absence of explicit dependence on the number of machines $r$. This is because we focus on the generalization error of the global output model, where the training data from all machines are utilized collectively—thus, replacing a single sample yields the same effect regardless of its originating machine. This naturally leads to a bound independent of $r$. Nevertheless, we hope that more refined analytical tools in the future could reveal how $r$ may influence the generalization behavior even in the global model setting. As a complementary perspective, we also analyze the generalization of local models in Section \ref{local model}, where the techniques differ and $r$-related effects become explicit.
\end{remark}
The learning rate plays a crucial role in \cref{convex case theorem 1}, as a smaller learning rate helps alleviate the adverse impact of decentralization on stability. To delve deeper into this relationship, we analyze the stability properties for two commonly used learning rate settings in the following corollary.
\begin{corollary}\label{corollary 1}
	Suppose $f(\mw;Z)$ is convex, satisfying \cref{Lipschitz,Smooth}.
	
	\textbf{(Constant Stepsize)} If the stepsize $\eta_t\equiv \eta\leq\frac{2}{\beta}$, the pointwise uniform stability we get $\varDelta_{rk}={2L^2\eta T}\lp\frac{2\eta\beta }{1-\lambda}+\frac{1}{mn}\rp$ and
	$
		\varDelta_{rk}^2=\frac{4L^4}{n}\sum_{k=1}^{n}\lp\frac{2\eta^2\beta T}{1-\lambda}+\frac{\eta}{m}\sum_{t=1}^{T}\mathbb{I}_{[j_t=k]}\rp^2.
	$
	
	\textbf{(Decreasing Stepsize)} If the stepsize $\eta_t=\frac{1}{(t+1)}$, the pointwise uniform stability can be expressed as: $\varDelta_{rk}={2L^2}\lp\frac{2\beta T}{T+1}+\frac{\ln \lp T+1\rp}{mn}\rp$ and 
	$
		\varDelta_{rk}^2=\frac{4L^4}{n}\sum_{k=1}^{n}\lp\frac{2\beta C_{\lambda}T}{T+1}+\frac{1}{m}\sum_{t=1}^{T}\frac{1}{t+1}\mathbb{I}_{[j_t=k]}\rp^2,
	$
	where $C_{\lambda}=\frac{1}{\lambda\log\frac{1}{\lambda}}\lp\frac{8}{e^{2}\log\frac{1}{\lambda}}+2\rp$ (See \cref{sum}).
\end{corollary}
\begin{remark}
	It is easy to check that under the first stepsize case, $\lk\varDelta_{rk}^2\rk^{1/2}$ is smaller that $\epsilon_{\textrm{uniform}}$, which means can get more refined generalization bound (The detailed proof can be found in \cref{compare}).
\end{remark}
\begin{theorem}[Generalization Error of D-SGD]\label{gen convex case}
	Assume the loss fuction $f(\mw;Z)$ is convex , $\beta$-smooth and L-Lipschitz. Let $\mw^{T+1}$ denote the weight parameter obtained after $T$ iterations of  D-SGD with the fixed stepsize $\eta_t=\eta$. For any $\delta\in(0,1)$, it holds with the probability at least $1-\delta$ that:
	\begin{align*}
		\Big|R(\mw^{T+1})-R_S(\mw^{T+1})\Big|
		\lesssim\frac{M\log^{\frac{1}{2}}(1/\delta)}{\sqrt{mn}}+\lk\varDelta_{rk}^2\rk^{1/2}\log mn\log (1/\delta),
	\end{align*}
	where 
	\begin{align*}
		\lk\varDelta_{rk}^2\rk^{1/2}\leq\frac{4\sqrt{2}L^2\eta^2\beta T}{1-\lambda}
		+\frac{2\sqrt{2}L^2\eta}{m}\lp\frac{1}{n}\sum_{k=1}^{n}\lp\sum_{t=1}^{T}\mathbb{I}_{[j_t=k]}\rp^2\rp^{\frac{1}{2}}.
	\end{align*}
	If $\eta=\sqrt{\frac{1}{T}}$ and $T=mn$, we can get $\Big|R(\A(S))-R_S(\A(S))\Big|=\mathcal{O}\lp\frac{1}{\sqrt{mn}}\rp$.
\end{theorem}
\cref{convex case theorem 1} establishes the pointwise uniform stability for the final iterate of D-SGD. However, the optimization error in the general convex case of D-SGD is typically expressed using the following average weight \citep{lian2017can,sun2021stability,bars2023improved}:
$
	\bar{\mw}^{T+1}={\sum_{t=1}^{T+1}\eta_t\mw_{t}}/{\sum_{t=1}^{T+1}\eta_t}.
$
\begin{proposition}
	Suppose $f(\mw;Z)$ is convex and Assumption \ref{Lipschitz} and \ref{Smooth} holds.
	\begin{itemize}
		\item 	If the stepsize $\eta_t\equiv \eta\leq\frac{2}{\beta}$, the average pointwise uniform stability we get $\varDelta_{rk}={2L^2\eta T}\lp\frac{\eta\beta}{1-\lambda}+\frac{1}{mn}\rp$ and
		$
		\varDelta_{rk}^2=\frac{4L^4\eta^2}{n}\sum_{k=1}^{n}\lp\frac{\eta\beta T}{1-\lambda}+\frac{1}{(T+1)m}\sum_{t=1}^{T}(T+1-t)\mathbb{I}_{[j_t=k]}\rp^2.
		$
		\item If the stepsize $\eta_t=\frac{1}{(t+1)}$, the average pointwise uniform stability we obtain
		$
		\varDelta_{rk}=4L^2\beta C_{\lambda}+\frac{L^2\ln (T+2)}{mn}.
		$
	\end{itemize}

\end{proposition}
\subsection{Strongly Convex Case}
Next, we turn to the strongly convex case. Let $\mathrm{P}_{\mathcal{W}} $ denote the Euclidean projection of a point onto the convex compact set $\mathcal{W}$. In this setting, the stochastic update in Algorithm 1 is replaced by its projected form:
\begin{align*}
	\mathbf{w}^{t+{1}}(i)=\mathrm{P}_{\mathcal{W}}\lk \mathbf{w}^{t+\frac{1}{2}}(i)-{\eta_t}\nabla f(\mathbf{w}^{t}(i);Z_{j_{t}(i)})\rk.
\end{align*}
\begin{definition}\label{Strong convex}
	$f(\mw;Z)$ is $\mu$-strongly convex if for any $Z$ and $\mw,\tilde{\mw}\in\mathcal{W}$,
	\begin{equation}
		f(\mw;Z) \geq f(\tilde{\mw};Z)+\langle\nabla f(\tilde{\mw};Z), \mw-\tilde{\mw}\rangle+\frac{\mu}{2}\|\mw-\tilde{\mw}\|_2^{2}.
	\end{equation}
	Specially, $f(\mw;Z)$ is convex if $\mu = 0$.
\end{definition}

\begin{theorem}\label{strong convex theorem 1}
	Let $f(\mw;Z)$ be $\mu$-strongly convex, L-Lipschitz and $\beta$-smooth for any Z. If the stepsizse $\eta_t\leq1/\beta$, then D-SGD with $T$ iterations achieves the pointwise uniform stability in fuction values, given by 
	\begin{align*}
		\varDelta_{rk}=
		4\beta L^2\sum_{t=1}^{T}\eta_t\sum_{j=1}^{t-1}\eta_j\lambda^{t-j-1}\prod_{\tilde{t}=t+1}^{T}\lp1-\frac{\eta_{\tilde{t}}\mu}{2}\rp+\frac{4L^2}{mn\mu}
	\end{align*}
	and
	\begin{align*}
		\varDelta_{rk}^2=\frac{4L^4}{n}\sum_{k=1}^{n}
		\lp\sum_{t=1}^{T}\lp2\beta \eta_t\sum_{q=1}^{t-1}\eta_q\lambda^{t-q-1}+\frac{1}{m}\eta_t\mathbb{I}_{[j_t=k]}\rp
		\prod_{\tilde{t}=t+1}^{T}\lp1-\frac{\eta_{\tilde{t}}\mu}{2}\rp\rp^2.
	\end{align*}
\end{theorem}
\begin{corollary}
	Suppose $f(\mw;z)$ is $\mu$-strongly convex and Assumption \ref{Lipschitz} and \ref{Smooth} hold.
	If the stepsize $\eta_t\equiv \eta\leq\frac{1}{\beta}$, the pointwise uniform stability we get
	$
	\varDelta_{rk}=\frac{4L^2}{\mu }\lp\frac{2\eta\beta}{1-\lambda}+\frac{1}{mn}\rp.
	$
	If the stepsize $\eta_t=\frac{1}{\mu(t+1)}$, the pointwise uniform stability we obtain
	$
	\varDelta_{rk}=\frac{16\beta L^2C_{\lambda}}{T\mu^2}\lp\ln T+1\rp+\frac{4L^2}{\mu mn}.
	$
\end{corollary}
\begin{remark}
	In the strongly convex setting, the first stepsize scheme yields a smaller value of $\lk\varDelta_{rk}^2\rk^{1/2}$ compared to $\epsilon_{\textrm{uniform}}$, indicating a tighter generalization bound. Detailed comparisons can be found in \cref{compare 11}.
\end{remark}
\begin{remark}
	It is not difficult to observe that the pointwise stability result in the strongly convex case under constant stepsize is independent of the number of iterations $T$, and the obtain results are also tighter, which matches the centralized states \citep{hardt2016train,lei2020fine}.
\end{remark}
\subsection{Nonconvex Case}
This section addresses the issues in the non-convex setting, focusing on the generalization error and optimization error. 
\begin{theorem}\label{nonconvex case theorem 1}
	Suppose that \cref{Lipschitz} and \cref{Smooth} holds. If we run D-SGD after $T$ iterations,
	then the pointwise uniform stability satisfies with
	$$
	\varDelta_{rk}=4\beta L^2\sum_{t=1}^{T}\eta_t\sum_{j=1}^{t-1}\eta_j\lambda^{t-j-1}\prod_{\tilde{t}=t+1}^{T}\lp1+\beta\eta_{\tilde{t}}\rp+\frac{2L^2}{mn}\sum_{t=1}^{T}\eta_t\prod_{\tilde{t}=t+1}^{T}\lp1+\beta\eta_{\tilde{t}}\rp
	$$
	and
	\begin{align*}
		\varDelta_{rk}^2
		=\frac{4L^4}{n}\sum_{k=1}^{n}\lp\sum_{t=1}^{T}\lp2\beta \eta_t\sum_{q=1}^{t-1}\eta_q\lambda^{t-q-1}+\frac{1}{m}\eta_t\mathbb{I}_{[j_t=k]}\rp\prod_{\tilde{t}=t+1}^{T}\lp1+\beta\eta_{\tilde{t}}\rp\rp^2.
	\end{align*}
\end{theorem}
\begin{remark}
	In the absence of convexity, this bound expands with the factor $1+\beta\eta$ (\cref{non} cannot be used), indicating that the algorithm stability deteriorates when optimizing non-convex problems. The result also highlight the impact of decentralization and learning rate on the stability of D-SGD. By using the smaller learning rates and constructing a small $\lambda$ communication graph, this effect can be mitigated \citep{sun2021stability,deng2023stability}.
\end{remark}
\begin{corollary}
	Assume that \cref{Lipschitz,Smooth} hold. If the stepsize $\eta_t\equiv\eta$, the pointwise uniform stability we get 
	$
		\varDelta_{rk}={2L^2}\lp\frac{2\eta}{1-\lambda}+\frac{1}{mn\beta}\rp\lp1+\beta\eta\rp^T.
	$
	If the stepsize $\eta_t=\frac{1}{(t+1)}$, the pointwise uniform stability we obtain 
	$
		\varDelta_{rk}=4L^2(T+1)^{\beta}\lp{4C_{\lambda}}+\frac{1}{\beta mn}\rp.
	$
\end{corollary}
Before deriving the optimization error, we need to introduce the following assumption.
\begin{assumption}[Gradient-dominance condition]
	\label{pl}
	If for any $\mw \subset \mathcal{W}$, there exists a constant $\gamma>0$ such that
	the following inequality holds:
	$
		R_S(\mathbf{w})-R_S\left(\mathbf{w}^{*}_{R}\right) \leq\frac{1}{4\gamma} \left\|\nabla R_S(\mathbf{w})\right\|_2^2,
	$
	where $\mathbf{w}^{*}_{R}=\arg \min_{\mw}R_S(w)$.
\end{assumption}
\begin{remark}
	The above condition, also known as the PŁ (Polyak-Łojasiewicz) condition, establishes a direct relationship between the value of the loss function and the norm of its gradient. This condition is widely used in the convergence and generalization analysis of non-convex optimization problems \citep{karimi2016linear,charles2018stability,foster2018uniform,lei2021learning,lei2021generalization}.
\end{remark}
\begin{theorem}\label{nonconvex optim}
	Suppose that \cref{Lipschitz,Smooth,Bound Variance} hold. Let $\{\mw^{t}\}_t$ be the sequence by D-SGD for $T$ iterations with $\eta_t=\eta\leq\frac{1}{3\beta}$. Then for any $\delta\in(0,1)$, with the probability at least $1-\delta$, we can get
	\begin{align*}
		\frac{\eta}{3}\sum_{t=1}^{T+1} \left\|\nabla R_S\left(\mathbf{w}^t\right)\right\|_2^2
		\leq&2 \log (2 / \delta) \max \left\{\eta L^2, \sigma^2 / \beta\right\}
		+\frac{3\beta L^2}{2}(T+1) \eta^4
		+\frac{3\beta \sigma^2 }{2}(T+1)\eta^2\\
		&+12L^2\beta\log (2 / \delta)
		+6\beta^3L^2\frac{(T+1)\eta^4}{(1-\lambda)^2}+2\beta L^2\frac{\eta^2(T+1)}{1-\lambda}.
	\end{align*}
	Furthermore, we can obtain the optimization error results of the average model by \cref{pl}
	\begin{align*}
		R_S(\bar{\mw}^{T+1})-R_S(\mw_R^{*})
		=\mathcal{O}\lp\frac{1}{T+1}\log (2 / \delta)+\frac{\eta^3}{(1-\lambda)^2}+\frac{\eta}{(1-\lambda)}\rp.
	\end{align*}
\end{theorem}
\begin{remark}
	Following the steps above, we can also bound the population gradient norm $\ls\nabla R(\mw)\rs$, which is crucial in non-convex problems \citep{ghadimi2013stochastic,foster2018uniform,davis2022graphical,lei2023stability,zhang2024generalization} where only local optima can be found. This is done via the following decomposition:$\ls\nabla R(\mw)\rs\leq{\ls\nabla R(\mw)-\nabla R_S(\mw)\rs}+{\ls\nabla R_S(\mw)\rs},$
	where the first term is covered by \cref{Lemma 1}.
\end{remark}

\begin{theorem}\label{nonconvex optim decrease}
	Assume that \cref{Lipschitz,Smooth,Bound Variance,pl} holds. Let $\{\mw^{t}\}_t$ be the sequence by D-SGD over $T$ iterations, where the stepsize is $\eta_t=\frac{2}{\gamma(t+1)}$. Then, for any $\delta\in(0,1)$, with the probability at least $1-\delta$, the following bound holds:
	\begin{align*}
		R_S({\mw}^{T+1})-R_S(\mw_R^{*})=\mathcal{O}\lp\frac{1}{T+1}\log (2 / \delta)\rp.
	\end{align*}
\end{theorem}
\section{Local Model}\label{local model}
In the previous section, we discussed the stability bounds and generalization error results for the averaged model $\mw^{T+1}=\frac{1}{m}\sum_{i=1}^{m}\mw^{T+1}(i)$. However, the global generalization bound does not capture the individual generalization performance of the local models at each client. Given the challenges in obtaining the averaged model, particularly when communication costs are high, it is crucial to investigate the generalization bounds for the local models instead. In the following, we explore the generalization performance of local models within the framework of the D-SGD algorithm with changing topology, a topic that has broad research significance, as highlighted in previous works \citep{nedic2017achieving,koloskova2020unified,le2023refined,hu2025stability}. We will also consider three distinct scenarios: convex, strongly convex [\cref{strong convex local}], and non-convex [\cref{non convex local}] settings. Before presenting our results, we need to introduce the definition of stability in the context of local models.
\begin{definition}[Local Pointwise Uniform Stability]\label{local model df}
	Let $\tilde{\epsilon}=\{\tilde{\epsilon}_1,\tilde{\epsilon}_2,\cdots,\tilde{\epsilon}_n\}$. We define $\A$ as $\tilde{\epsilon}$-local model ($r$-th) pointwise uniform stability for distributed learning in function values if  
	\begin{equation}
		\sup_{Z\in S_r}\left|f\lp\A(S);Z\rp-f\lp\A(S^{(rk)});Z\rp\right|\leq \tilde{\epsilon}_{k},
	\end{equation}
	where $S^{(rk)}$ is defined in \cref{pointwise uniform dl}.
\end{definition}

\begin{theorem}\label{convex local}
	Suppose that $f(\mw;Z)$ is convex, L-Lipschitz and $\beta$-smooth for any Z with respect to $\mw$. If the stepsize $\eta_t\leq{2P^{t}_{rr}}/{\beta}$, then the $r$-th local model pointwise uniform stability of time-varying D-SGD after $T$ iterations can be bounded by $\frac{1}{n}\sum_{k=1}^{n}\tilde{\epsilon}_k
	=\frac{2L^2}{n}\sum_{t=1}^{T}P^{T:t}_{rr}\eta_t$
	and 
	\begin{align*}
		\frac{1}{n}\sum_{k=1}^{n}\tilde{\epsilon}_{k}^2=\frac{4L^4}{n}\sum_{k=1}^{n}\lp\sum_{t=1}^{T}P^{T:t}_{rr}\eta_t\mathbb{I}_{[j_t=k]}\rp^2,
	\end{align*}
	where $P^{T:t}_{rr}=P^{T}_{rr}P^{T-1}_{rr}\cdots P^{t}_{rr}$ is the product of $T-t+1$ gossip matrices.
\end{theorem}
\begin{remark}
	It is vital to note that $P^{T:t}$ remains a doubly stochastic matrix, with $0<P^{T:t}_{rr}<1$.
	Using the stability bound established above, we can easily derive the upper bound for the local model generalization error. As for the difference between \cref{convex local} and \cref{corollary 1}, in \cref{convex local}, we do not impose additional structural assumptions on $P$; rather, its influence is considered only through the step-size condition. In \cref{corollary 1}, however, we aim to explore more properties of $P$ and thus adopt the consensus error analysis from \cite{sun2021stability} to obtain a broader result that explicitly reflects the impact of network topology on generalization.For more details, please refer to \cref{compare 1}.
	
\end{remark}

\section{Concluding Remarks}\label{conclusion} 
This paper addresses the high-probability generalization properties of D-SGD. We establish refined stability and generalization bounds for D-SGD under pointwise uniform stability, a weaker assumption than uniform stability. Our high-probability bounds achieve the sharp convergence rate of $\mathcal{O}(\frac{1}{\sqrt{mn}})$ and are applicable to convex, strongly convex, and non-convex scenarios. Additionally, we derive high-probability bounds for optimization error and excess risk in non-convex problems under gradient-dominance conditions, offering a deeper understanding of D-SGD's performance in challenging settings. Finally, we derive generalization error bounds for local models in time-varying D-SGD, highlighting the impact of dynamic topologies on model performance.
In future, it would be interesting to explore relaxing the Lipschitz assumption \citep{lei2020fine,zhu2022topology,lei2023mini}, or consider non-smooth optimization problems \citep{bassily2020non,fan2024high}. Additionally, investigating high-probability results for D-SGD without replacement \citep{lei2020fine,yuan2024l_2} and DM-SGD \citep{wang2024towards} are potential directions.

\noindent\textit{Challenges and Difficulties:}\label{challenges}

\textbf{ The gap from centralized to decentralized learning:}
The existing work in \cite{fan2024high} primarily focuses on centralized SGD. However, extending these results to decentralized settings is non-trivial due to fundamental differences in stability analysis. Specifically, in decentralized learning, pointwise uniform stability must account for additional complexity introduced by the network topology, including factors such as machine index 
$r$ and sample index $k$. A key question we aim to address is: Can we develop a similarly weak stability tool for decentralized settings, as in the centralized case? By answering this question, our work provides the first high-probability generalization bounds for decentralized learning, measured in terms of function values and gradients, thereby bridging an important theoretical gap in the field.

\textbf{ The innovations and challenges in stability analysis for D-SGD:}
A major challenge in analyzing the stability of D-SGD lies in handling the gossip matrix and the gap between local models and the aggregated model. Our approach builds upon the techniques of \cite{sun2021stability}, but with a crucial refinement: unlike their framework, we do not introduce an additional aggregated model term $w^t$, leading to a tighter error bound (See \cref{convex remark} in our paper for details).

Furthermore, introducing pointwise uniform stability in a decentralized setting complicates the choice of step sizes. This arises due to the presence of the indicator function $I_{[j_t(r) = k]}$, which makes the analysis of average weighting schemes more intricate. Additionally, differences in the assumptions on the loss function introduce variations in recursive expansion factors, adding further challenges when analyzing different models under varying step size choices.

\textbf{ Optimization Error in High-Probability Regimes:}
Unlike standard expectation-based optimization error analysis, we cannot directly approximate the empirical risk term using unbiased gradient estimates. In contrast to conventional SGD \citep{nemirovski2009robust} and D-SGD \citep{koloskova2020unified,sun2021stability} analyses, high-probability analysis requires the introduction of martingale difference sequence techniques to control error accumulation and bound the term $\sum_{t=1}^{T+1} \eta_t\left\|\nabla R_S\left(\mathbf{w}^t\right)\right\|_2^2$. Subsequently, by leveraging the gradient-dominance assumption, we derive the corresponding optimization error bounds.

\noindent\textit{Limitations:}\label{limitations}

\textbf{Assumption.} Our analysis relies on the Lipschitz continuity assumption. While some recent works have shown stability-based generalization benefits of optimization beyond the Lipschitz setting \citep{lei2020fine,zhu2022topology,lei2023mini,wang2024towards}, these results are mostly in expectation and not directly applicable to our setting. Moreover, the characterization of consensus error \citep{sun2021stability,deng2023stability,zhu2024stability} in our analysis critically depends on the Lipschitz property. Developing sharper tools to quantify consensus error under weaker assumptions remains an important direction. Additionally, establishing high-probability generalization bounds for D-SGD under non-smooth case also remains an open challenge.

\textbf{Asynchronous Setting.} While our analysis focuses on the synchronous setting, prior work such as \cite{deng2023stability} highlights the practical relevance of AD-SGD. Although we did not explicitly analyze the asynchronous case, our framework can be trivially extended to obtain weaker stability bounds in this setting, which can potentially support generalization results for AD-SGD.


\acks{This work was supported in part by the  National Natural Science Foundation of China (Grant Nos. 62522610,  62376278), and NUDT Foundational Research Funding (JS25-02).}


\appendix
\section{Technical Tools}
\begin{lemma}[Lemma 3.6, \cite{hardt2016train}]
	\label{non}
	Assume the loss function $f\left(\mw ; Z\right)$ is convex and $\beta$-smooth with respect to $\mw$ for all $Z$. Then for $\eta \leq 2 / \beta$ it holds
	\begin{equation*}
		\left\|\mathbf{w}-\eta \nabla f\left(\mathbf{w} ; Z\right)-\mathbf{w}^{\prime}+\eta \nabla f\left(\mathbf{w}^{\prime} ; Z\right)\right\|_{2} \leq\left\|\mathbf{w}-\mathbf{w}^{\prime}\right\|_{2} .
	\end{equation*}
	Furthermore, if the loss function $f\left(\mathbf{w} ; Z\right)$ is $\mu$-strongly convex and $\beta$-smooth, the stepsize satisfies $\eta\leq1/\beta$, then
	\begin{equation*}
		\left\|\mathbf{w}-\eta \nabla f\left(\mathbf{w} ; Z\right)-\mathbf{w}^{\prime}+\eta \nabla f\left(\mathbf{w}^{\prime} ; Z\right)\right\|_{2} \leq(1-\frac{\eta\mu}{2})\left\|\mathbf{w}-\mathbf{w}^{\prime}\right\|_{2} .
	\end{equation*}
\end{lemma}
\begin{lemma}[Lemma 5, \cite{sun2021stability,deng2023stability}]\label{sum}
	For any $0<\lambda<1$ and $t\in Z^+$, we have
	\begin{equation*}
		\sum_{q=1}^{t-1}\frac{\lambda^{t-1-q}}{q+1}\leq\frac{C_{\lambda}}{t}, \quad\quad C_{\lambda}:=\frac{1}{\lambda\log\frac{1}{\lambda}}\lp\frac{8}{e^{2}\log\frac{1}{\lambda}}+2\rp.
	\end{equation*}
\end{lemma}
\begin{lemma}[Lemma 8, \cite{sun2021stability}]
	Suppose that \cref{Lipschitz} holds, and $\{\mw^{t+1}(i),{\mw}^{t+1}\}$ are generated by D-SGD during the $t$-th iteration. In this case, the difference between the average model ${\mw}^{t+1}$ and each local model $\mw^{t+1}(i)$ can be bounded as follows:
	\begin{equation*}
		\lk\sum_{i=1}^{m}\ls{\mw}^{t+1}-\mw^{t+1}(i)\rs^2\rk^\frac{1}{2}\leq2\sqrt{m}L\sum_{q=1}^{t}\eta_q\lambda^{t-q}.
	\end{equation*}
\end{lemma}
\begin{lemma}[\cite{vershynin2018high,lei2021learning}]
	\label{mag}
	Let $z_1, \ldots, z_m$ be a sequence of random variables, where each $z_i$ may depend on the preceding random variables $z_1, \ldots, z_{i-1}$ for all $i=1, \ldots, m$. Consider a sequence of functionals $\xi_i\left(z_1, \ldots, z_i\right)$and define the conditional variance as: $\sigma_m^2=\sum_{i=1}^m \mathbb{E}_{z_i}\lk\left(\xi_i-\mathbb{E}_{z_i}\left[\xi_i\right]\right)^2\rk$.
	
	(i) Assume that $\left|\xi_i-\mathbb{E}_{z_i}\left[\xi_i\right]\right| \leq b_i$ for each $i$. For $\delta \in(0,1)$, with probability at least $1-\delta$, we have:
	\begin{align*}
		\sum_{i=1}^m \xi_i-\sum_{i=1}^m\mathbb{E}_{z_i}\left[\xi_i\right] \leq\left(2 \sum_{i=1}^m b_i^2 \log \frac{1}{\delta}\right)^{\frac{1}{2}}.
	\end{align*}
	
	(ii) Suppose that $\xi_i-\mathbb{E}_{z_i}\left[\xi_i\right] \leq b$ for each $i$. Let $\rho \in(0,1]$ and $\delta \in(0,1)$. With probability at least $1-\delta$, the following bound holds
	\begin{align*}
		\sum_{i=1}^m \xi_i-\sum_{i=1}^m \mathbb{E}_{z_i}\left[\xi_i\right] \leq \frac{\rho \sigma_m^2}{b}+\frac{b \log \frac{1}{\delta}}{\rho}.
	\end{align*}
\end{lemma}

\section{Convex Case}\label{section Convex Case}
Before proceeding with the proof, let's first review the update details of D-SGD and try to derive some properties.
Consider the output for the algorithm,
\begin{align*}
	{\mw}^{t+1}=\frac{1}{m}\sum_{i=1}^{m}\mw^{t+1}(i)=&\frac{1}{m}\sum_{i=1}^{m}\lk\sum_{l=1}^{m}P_{i,l}\mw^{t}(l)-\eta_t\nabla f\lp\mw^{t}(i);Z_{j_t(i)}\rp\rk\\
	\overset{\textbf{(a)}}{=}&\frac{1}{m}\sum_{l=1}^{m}\mw^{t}(l)-\frac{\eta_t}{m}\sum_{i=1}^{m}\nabla f\lp\mw^{t}(i);Z_{j_t(i)}\rp\\
	=&{\mw}^{t}-\frac{\eta_t}{m}\sum_{i=1}^{m}\nabla f\lp\mw^{t}(i);Z_{j_t(i)}\rp,
\end{align*}
where Equation $\textbf{(a)}$ uses the fact of gossip matrix property (See \cref{gossip}). Then we get
\begin{equation}
	{\mw}^{t+1}={\mw}^{t}-\frac{\eta_t}{m}\sum_{i=1}^{m}\nabla f\lp\mw^{t}(i);Z_{j_t(i)}\rp.
\end{equation}

\subsection{Proof of \cref{convex case theorem 1}}\label{B.1}
\begin{proof}
	
	Recall the definition of pointwise uniform stability, we need to measure the gap or distance between two sets of weights in two datasets that differs by only one sample. Let ${\mw}^{t+1}$ and ${\mv}^{t+1}$ be produced by D-SGD based on $S$ and $S_{rk}$ (See \cref{pointwise uniform dl}) respectively. Specifically, these sets are constructed as follows:
	\begin{align*}
		S=&\{Z_{1(1)},\cdots,Z_{n(1)},\cdots,Z_{1(r)},\cdots,Z_{k(r)},\cdots,Z_{n(r)},\cdots,Z_{1(m)},\cdots,Z_{n(m)}\},\\
		S_{rk}=&\{Z_{1(1)},\cdots,Z_{n(1)},\cdots,Z_{1(r)},\cdots,\tilde{Z}_{k(r)},\cdots,Z_{n(r)},\cdots,Z_{1(m)},\cdots,Z_{n(m)}\}.
	\end{align*}
	Here, $\tilde{Z}_{k(r)}$ denotes the altered data in the $r$-th subset.
	
	Review the D-SGD update step,
	\begin{equation*}
		{\mw}^{t+1}(i)=\sum_{l=1}^{m}P_{i,l}\mw^{t}(l)-\eta_t\nabla f\lp\mw^{t}(i);Z_{j_t(i)}\rp.
	\end{equation*}
	If Change-One dataset $S_{rk}$ is used,
	\begin{equation*}
		{\mv}^{t+1}(i)=\sum_{l=1}^{m}P_{i,l}\mv^{t}(l)-\eta_t\nabla f\lp\mv^{t}(i);\tilde{Z}_{j_t(i)}\rp.
	\end{equation*}
	Consider the global weight difference, with the probability $1-\frac{1}{n}$,
	\begin{align*}
		&\ls{\mw}^{t+1}-{\mv}^{t+1}\rs\\
		=&\ls\frac{1}{m}\sum_{i=1}^{m}\mw^{t+1}(i)-\frac{1}{m}\sum_{i=1}^{m}\mv^{t+1}(i)\rs\\
		=&\frac{1}{m}\sum_{i=1}^{m}\ls\sum_{l=1}^{m}P_{i,l}\mw^{t}(l)-\eta_t\nabla f\lp\mw^{t}(i);Z_{j_t(i)}\rp-\sum_{l=1}^{m}P_{i,l}\mv^{t}(l)+\eta_t\nabla f\lp\mv^{t}(i);\tilde{Z}_{j_t(i)}\rp\rs\\
		=&\frac{1}{m}\sum_{i=1}^{m}\ls\sum_{l=1}^{m}P_{i,l}\mw^{t}(l)-\eta_t\nabla f\lp\mw^{t}(i);Z_{j_t(i)}\rp-\sum_{l=1}^{m}P_{i,l}\mv^{t}(l)+\eta_t\nabla f\lp\mv^{t}(i);{Z}_{j_t(i)}\rp\rs\\
		=&\frac{1}{m}\sum_{i=1}^{m}\Bigg\|\sum_{l=1}^{m}P_{i,l}\mw^{t}(l)-\eta_t\nabla f\lp\mw^{t}(i);Z_{j_t(i)}\rp+\eta_t\nabla f\lp\mw^{t};Z_{j_t(i)}\rp-\eta_t\nabla f\lp\mw^{t};Z_{j_t(i)}\rp\\
		&\quad\quad\quad\quad-\sum_{l=1}^{m}P_{i,l}\mv^{t}(l)+\eta_t\nabla f\lp\mv^{t}(i);Z_{j_t(i)}\rp+\eta_t\nabla f\lp\mv^{t};Z_{j_t(i)}\rp-\eta_t\nabla f\lp\mv^{t};Z_{j_t(i)}\rp\Bigg\|_2\\
		\leq&\frac{1}{m}\sum_{i=1}^{m}\ls\sum_{l=1}^{m}P_{i,l}\mw^{t}(l)-\eta_t\nabla f\lp\mw^{t};Z_{j_t(i)}\rp-\sum_{l=1}^{m}P_{i,l}\mv^{t}(l)+\eta_t\nabla f\lp\mv^{t};Z_{j_t(i)}\rp\rs\\
		&\quad\quad+\frac{1}{m}\sum_{i=1}^{m}\ls\eta_t\nabla f\lp\mw^{t};Z_{j_t(i)}\rp-\eta_t\nabla f\lp\mw^{t}(i);Z_{j_t(i)}\rp\rs\\
		&\quad\quad+\frac{1}{m}\sum_{i=1}^{m}\ls\eta_t\nabla f\lp\mv^{t}(i);Z_{j_t(i)}\rp-\eta_t\nabla f\lp\mv^{t};Z_{j_t(i)}\rp\rs\\
		\leq
		&\frac{1}{m}\ls\sum_{i=1}^{m}\sum_{l=1}^{m}P_{i,l}\mw^{t}(l)-\sum_{i=1}^{m}\eta_t\nabla f\lp\mw^{t};Z_{j_t(i)}\rp-\sum_{i=1}^{m}\sum_{l=1}^{m}P_{i,l}\mv^{t}(l)+\sum_{i=1}^{m}\eta_t\nabla f\lp\mv^{t};Z_{j_t(i)}\rp\rs\\
		&\quad\quad+\frac{\eta_t\beta}{m}\sum_{i=1}^{m}\ls\mw^{t}-\mw^{t}(i)\rs+\frac{\eta_t\beta}{m}\sum_{i=1}^{m}\ls\mv^{t}-\mv^{t}(i)\rs.\\
		\overset{\textbf{(a)}}{\leq}&\frac{1}{m}\ls\sum_{l=1}^{m}\mw^{t}(l)-\sum_{i=1}^{m}\eta_t\nabla f\lp\mw^{t};Z_{j_t(i)}\rp-\sum_{l=1}^{m}\mv^{t}(l)+\sum_{i=1}^{m}\eta_t\nabla f\lp\mv^{t};Z_{j_t(i)}\rp\rs\\
		&\quad\quad+\frac{\eta_t\beta}{m}\sum_{i=1}^{m}\ls\mw^{t}-\mw^{t}(i)\rs+\frac{\eta_t\beta}{m}\sum_{i=1}^{m}\ls\mv^{t}-\mv^{t}(i)\rs\\
		\overset{\textbf{(b)}}{\leq}&\ls\mw^{t}-\frac{1}{m}\sum_{i=1}^{m}\eta_t\nabla f\lp\mw^{t};Z_{j_t(i)}\rp-\mv^{t}+\frac{1}{m}\sum_{i=1}^{m}\eta_t\nabla f\lp\mv^{t};Z_{j_t(i)}\rp\rs\\
		&\quad\quad+\frac{\eta_t\beta}{\sqrt{m}}\lk\sum_{i=1}^{m}\ls{\mw}^{t}-\mw^{t}(i)\rs^2\rk^\frac{1}{2}+\frac{\eta_t\beta}{\sqrt{m}}\lk\sum_{i=1}^{m}\ls{\mv}^{t}-\mv^{t}(i)\rs^2\rk^\frac{1}{2}\\
		\overset{\textbf{(c)}}{\leq}&\ls{\mw}^{t}-{\mv}^{t}\rs+4\eta_t\beta L\sum_{q=1}^{t-1}\eta_q\lambda^{t-q-1},
	\end{align*}
	where the inequality $\textbf{(a)}$  the $\beta$-smoothness and the doubly random matrix property, the inequality $\textbf{(b)}$ relies on the basic inequality, and the inequality $\textbf{(c)}$ employs the non-expansive operator (see \cref{non}).
	
	With the probability $\frac{1}{n}$, if $j_t=k$ (anomalous sample $\tilde{Z}_{k(r)}$ is selected during the update process), 
	\begin{align*}
		&\ls{\mw}^{t+1}-{\mv}^{t+1}\rs\\
		\leq&\frac{1}{m}\sum_{i=1}^{m}\ls\sum_{l=1}^{m}P_{i,l}\mw^{t}(l)-\eta_t\nabla f\lp\mw^{t};Z_{k(i)}\rp-\sum_{l=1}^{m}P_{i,l}\mv^{t}(l)+\eta_t\nabla f\lp\mv^{t};\tilde{Z}_{k(i)}\rp\rs\\
		&+\frac{1}{m}\sum_{i=1}^{m}\ls\eta_t\nabla f\lp\mw^{t};Z_{k(i)}\rp-\eta_t\nabla f\lp\mw^{t}(i);Z_{k(i)}\rp\rs\\
		&+\frac{1}{m}\sum_{i=1}^{m}\ls\eta_t\nabla f\lp\mv^{t}(i);\tilde{Z}_{k(i)}\rp-\eta_t\nabla f\lp\mv^{t};\tilde{Z}_{k(i)}\rp\rs\\
		\leq&\frac{1}{m}\sum_{i=1}^{m}\ls\sum_{l=1}^{m}P_{i,l}\mw^{t}(l)-\eta_t\nabla f\lp\mw^{t};Z_{k(i)}\rp-\sum_{l=1}^{m}P_{i,l}\mv^{t}(l)+\eta_t\nabla f\lp\mv^{t};\tilde{Z}_{k(i)}\rp\rs\\
		&+4\eta_t\beta L\sum_{q=1}^{t-1}\eta_q\lambda^{t-q-1}\\
		\overset{\textbf{(d)}}{\leq}&\ls\mw^{t}-\frac{\eta_t}{m}\sum_{i=1,i\neq r}^{m}\nabla f\lp\mw^{t};Z_{k(i)}\rp-\mv^t+\frac{\eta_t}{m}\sum_{i=1,i\neq r}^{m}\nabla f\lp\mv^{t};Z_{k(i)}\rp\rs\\
		&+\frac{\eta_t}{m}\ls\nabla f\lp\mw^{t};Z_{k(r)}\rp-\nabla f\lp\mv^{t};\tilde{Z}_{k(r)}\rp\rs+4\eta_t\beta L\sum_{q=1}^{t-1}\eta_q\lambda^{t-q-1}\\
		\overset{\textbf{(e)}}{\leq}&\ls{\mw}^{t}-{\mv}^{t}\rs+4\eta_t\beta L\sum_{q=1}^{t-1}\eta_q\lambda^{t-q-1}+\frac{2\eta_tL}{m}.
	\end{align*}
	For the above inequality $\textbf{(d)}$, we separate the machine $r$ to determine whether anomalous samples exist. This ensures that the former can continue leveraging the non-expansive operator property (See \cref{non}). Inequality $\textbf{(e)}$, on the other hand, is derived based on the L-Lipschitz assumption (See \cref{Lipschitz}).
	
	Combining the two cases mentioned above, we further derive the following result
	\begin{equation}
		\ls{\mw}^{t+1}-{\mv}^{t+1}\rs\leq\ls{\mw}^{t}-{\mv}^{t}\rs+4\eta_t\beta L\sum_{q=1}^{t-1}\eta_q\lambda^{t-q-1}+\frac{2\eta_tL}{m}\mathbb{I}_{[j_t=k]},
	\end{equation}
	where $\mathbb{I}_{[j_t=k]}$ denotes the event that machine selects the $k$-th sample at step $t$.
	Apply the above inquality recursively to obtain
	\begin{equation}\label{convex recursion}
		\ls{\mw}^{T+1}-{\mv}^{T+1}\rs\leq4\beta L\sum_{t=1}^{T}\eta_t\sum_{q=1}^{t-1}\eta_q\lambda^{t-q-1}+\frac{2L}{m}\sum_{t=1}^{T}\eta_t\mathbb{I}_{[j_t=k]}.
	\end{equation}
	Based on the definition of $\epsilon$-pointwise stability (See \cref{pointwise uniform dl}) and the Lipschitz property of the loss function, we can conclude that
	\begin{align*}
		sup_{Z}\left|f\lp{\mw}^{T+1};Z\rp-f\lp{\mv}^{T+1};Z\rp\right|
		&\leq L \ls{\mw}^{T+1}-{\mv}^{T+1}\rs\\
		&\leq4\beta L^2\sum_{t=1}^{T}\eta_t\sum_{q=1}^{t-1}\eta_q\lambda^{t-q-1}+\frac{2L^2}{m}\sum_{t=1}^{T}\eta_t\mathbb{I}_{[j_t=k]}\\
		&=\epsilon_{rk}.
	\end{align*}
	It the follows that
	\begin{align*}
		\frac{1}{mn}\sum_{r=1}^{m}\sum_{k=1}^{n}\epsilon_{rk}^2=\frac{4L^4}{n}\sum_{k=1}^{n}\lp2\beta \sum_{t=1}^{T}\eta_t\sum_{q=1}^{t-1}\eta_q\lambda^{t-q-1}+\frac{1}{m}\sum_{t=1}^{T}\eta_t\mathbb{I}_{[j_t=k]}\rp^2
	\end{align*}
	and 
	\begin{align*}
		\frac{1}{mn}\sum_{r=1}^{m}\sum_{k=1}^{n}\epsilon_k=&\frac{2L^2}{mn}\sum_{r=1}^{m}\sum_{k=1}^{n}\lp2\beta \sum_{t=1}^{T}\eta_t\sum_{q=1}^{t-1}\eta_q\lambda^{t-q-1}+\frac{1}{m}\sum_{t=1}^{T}\eta_t\mathbb{I}_{[j_t=k]}\rp\\
		=&4\beta L^2\sum_{t=1}^{T}\eta_t\sum_{q=1}^{t-1}\eta_q\lambda^{t-q-1}+\frac{2L^2}{mn}\sum_{t=1}^{T}\sum_{k=1}^{n}\eta_t\mathbb{I}_{[j_t=k]}\\
		=&4\beta L^2\sum_{t=1}^{T}\eta_t\sum_{q=1}^{t-1}\eta_q\lambda^{t-q-1}+\frac{2L^2}{mn}\sum_{t=1}^{T}\eta_t,
	\end{align*}
	where the last equality uses $\sum_{k=1}^{n}\mathbb{I}_{[j_t=k]}=1$.
\end{proof}
\begin{remark}\label{convex remark}
	Now, let's compare {Pointwise Uniform Stability} with {Uniform Stability} in the 
	{Convex Case}.
	Based on \cref{convex recursion}, we can also derive the result for the uniform stability of the D-SGD algorithm, specifically:
	\begin{align}
		\mathbb{E}[\epsilon_{\emph{uniform}}]=4\beta L^2\sum_{t=1}^{T}\eta_t\sum_{q=1}^{t-1}\eta_q\lambda^{t-q-1}+\frac{2L^2}{mn}\sum_{t=1}^{T}\eta_t\label{U1}.
	\end{align}
	Alternatively, it can also be written as:
	\begin{align*}
		\epsilon_{\emph{uniform}}=4\beta L^2\sum_{t=1}^{T}\eta_t\sum_{q=1}^{t-1}\eta_q\lambda^{t-q-1}+\frac{2L^2}{m}\max_{k\in[n]}\sum_{t=1}^{T}\eta_t\mathbb{I}_{[j_t=k]}.
	\end{align*}
	For comparison, the result from \cite{sun2021stability} is given by:
	\begin{align*}
		\mathbb{E}[\epsilon_{\emph{uniform}}]=4 L^2\sum_{t=1}^{T}\lp1+\eta_t\beta\rp\sum_{q=1}^{t-1}\eta_q\lambda^{t-q-1}+\frac{2L^2}{mn}\sum_{t=1}^{T}\eta_t.
	\end{align*}
	Note that Theorem 1 from \cite{sun2021stability} contains a minor typo. After verification, we have presented the corrected result above for comparison. We can observe that our result is more compact than the one in \cite{sun2021stability} due to the absence of an extra term, making it tighter. Additionally, we have provided two different forms of the uniform stability result, which will facilitate a comparison between Uniform Stability (US) and Pointwise Uniform Stability (PUS).
\end{remark}

\subsection{Discussion about learning rate}
\label{convex case lr}
\textbf{Case I [Constant stepsize]:}
\begin{proof}
	When $\eta_t\equiv \eta\leq\frac{2}{L}$, then we can get
	\begin{equation}
		\ls{\mw}^{t+1}-{\mv}^{t+1}\rs\leq\ls{\mw}^{t}-{\mv}^{t}\rs+\frac{4\eta^2\beta L}{1-\lambda}+\frac{2\eta L}{m}\mathbb{I}_{[j_t=k]}.
	\end{equation}
	By iteratively employing the aforementioned inequality, we can derive
	\begin{align}
		\ls{\mw}^{T+1}-{\mv}^{T+1}\rs\leq&\sum_{t=1}^{T}\lp\frac{4\eta^2\beta L}{1-\lambda}+\frac{2\eta L}{m}\mathbb{I}_{[j_t=k]}\rp\nonumber\\
		=&\frac{4\eta^2\beta LT}{1-\lambda}+\frac{2\eta L}{m}\sum_{t=1}^{T}\mathbb{I}_{[j_t=k]}\label{convex constant}.
	\end{align}
	Then
	\begin{align*}
		\epsilon_{rk}=&\frac{4\eta^2\beta L^2T}{1-\lambda}+\frac{2\eta L^2}{m}\sum_{t=1}^{T}\mathbb{I}_{[j_t=k]}\\
		=&2L^2\lp\frac{2\eta^2\beta T}{1-\lambda}+\frac{\eta}{m}\sum_{t=1}^{T}\mathbb{I}_{[j_t=k]}\rp.
	\end{align*}
	It then follows that 
	\begin{align*}
		\frac{1}{mn}\sum_{r=1}^{m}\sum_{k=1}^{n}\epsilon_{rk}^2=\frac{4L^4}{n}\sum_{k=1}^{n}\lp\frac{2\eta^2\beta T}{1-\lambda}+\frac{\eta}{m}\sum_{t=1}^{T}\mathbb{I}_{[j_t=k]}\rp^2
	\end{align*}
	and
	\begin{align*}
		\frac{1}{mn}\sum_{r=1}^{m}\sum_{k=1}^{n}\epsilon_{rk}=&\frac{2L^2}{ mn}\sum_{r=1}^{m}\sum_{k=1}^{n}\lp\frac{2\eta^2\beta T}{1-\lambda}+\frac{\eta}{m}\sum_{t=1}^{T}\mathbb{I}_{[j_t=k]}\rp\\
		=&{2L^2\eta T}\lp\frac{2\eta\beta }{1-\lambda}+\frac{1}{mn}\rp.
	\end{align*}
	Moreover,
	\begin{align*}
		\lp\frac{1}{mn}\sum_{r=1}^{m}\sum_{k=1}^{n}\epsilon_{rk}^2\rp^{\frac{1}{2}}=&\lp\frac{4L^4}{n}\sum_{k=1}^{n}\lp\frac{2\eta^2\beta T}{1-\lambda}+\frac{\eta}{m}\sum_{t=1}^{T}\mathbb{I}_{[j_t=k]}\rp^2\rp^{\frac{1}{2}}\\
		=&{2L^2}\lp\frac{1}{n}\sum_{k=1}^{n}\lp\frac{2\eta^2\beta T}{1-\lambda}+\frac{\eta}{m}\sum_{t=1}^{T}\mathbb{I}_{[j_t=k]}\rp^2\rp^{\frac{1}{2}}\\
		\overset{\textbf{(a)}}{\leq}&{2L^2}\lk\frac{1}{n}\sum_{k=1}^{n}\lp2\lp\frac{2\eta^2\beta T}{1-\lambda}\rp^2+2\lp\frac{\eta}{m}\sum_{t=1}^{T}\mathbb{I}_{[j_t=k]}\rp^2\rp\rk^{\frac{1}{2}}\\
		\leq&{2L^2}\lk8\lp\frac{\eta^2\beta T}{1-\lambda}\rp^2+\frac{2}{n}\sum_{k=1}^{n}\lp\frac{\eta}{m}\sum_{t=1}^{T}\mathbb{I}_{[j_t=k]}\rp^2\rk^{\frac{1}{2}}\\
		\overset{\textbf{(b)}}{\leq}&\frac{4\sqrt{2}L^2\eta^2\beta T}{1-\lambda}+\frac{2\sqrt{2}L^2\eta}{m}\lp\frac{1}{n}\sum_{k=1}^{n}\lp\sum_{t=1}^{T}\mathbb{I}_{[j_t=k]}\rp^2\rp^{\frac{1}{2}},
	\end{align*}
	where Equation \textbf{(a)} uses the elementary inequality $(x+y)^2\leq2x^2+2y^2$ and Equation \textbf{(b)} relies on $\sqrt{x+y}\leq\sqrt{x}+\sqrt{y}$.
\end{proof}
\begin{remark}\label{compare}
	Now, we compare the US and PUS under convex case. If $T=n$, we can get
	\begin{align*}
		\lp\frac{1}{n}\sum_{k=1}^{n}\lp\sum_{t=1}^{n}\mathbb{I}_{[j_t=k]}\rp^2\rp^{\frac{1}{2}}	\leq\lp\max_{k\in[n]}\sum_{t=1}^{n}\mathbb{I}_{[j_t=k]}\rp^{\frac{1}{2}}.
	\end{align*}
	Based on the above inequality \citep{fan2024high}, it is easy to verify that 
	$\lk\varDelta_{rk}^2\rk^{1/2}\leq\epsilon_{\emph{uniform}}$.
\end{remark}
\begin{remark}[More details about Introduction]
	By combining \cref{U1} and the relationship between stability and generalization in the expectation sense \citep{hardt2016train,sun2021stability}, we can derive:
	\begin{align*}
		\Big| \mathbb{E}_{S,\A}\lk R_S(\A(S))-R(\A(S))\rk\Big|\leq4\beta L^2\sum_{t=1}^{T}\eta_t\sum_{q=1}^{t-1}\eta_q\lambda^{t-q-1}+\frac{2L^2}{mn}\sum_{t=1}^{T}\eta_t.
	\end{align*}
	If $\eta_t=\eta$,
	\begin{equation}
		\ls{\mw}^{t+1}-{\mv}^{t+1}\rs\leq\ls{\mw}^{t}-{\mv}^{t}\rs+\frac{4\eta^2\beta L}{1-\lambda}+\frac{2\eta L}{mn}.
	\end{equation}
	Then we can get 
	\begin{align*}
		\epsilon_{\emph{uniform}}=\frac{4\eta^2\beta L^2T}{1-\lambda}+\frac{2T\eta L^2}{mn}
		=2L^2\lp\frac{2\eta^2\beta T}{1-\lambda}+\frac{\eta T}{mn}\rp,
	\end{align*}
	and
	\begin{align*}
		\Big| \mathbb{E}_{S,\A}\lk R_S(\A(S))-R(\A(S))\rk\Big|\leq2L^2\lp\frac{2\eta^2\beta T}{1-\lambda}+\frac{\eta T}{mn}\rp.
	\end{align*}
	According to Markov's inequality, with $1-\delta$ probability, we obtain
	\begin{align*}
		R(\A(S))-R_S(\A(S))\leq\frac{2L^2}{\delta}\lp\frac{2\eta^2\beta T}{1-\lambda}+\frac{\eta T}{mn}\rp.
	\end{align*}
	If $\eta=\frac{1}{\sqrt{T}}$ and $T=mn$, we can get 
	\begin{equation*}
		R(\A(S))-R_S(\A(S))=\mathcal{O}\lp\frac{1}{\delta\sqrt{mn}}\rp.
	\end{equation*} 
\end{remark}

\noindent\textbf{Case II [Decreasing stepsize]:}
\begin{proof}
	When $\eta_t=\frac{1}{(t+1)}$, then we can get
	\begin{align}
		\ls{\mw}^{t+1}-{\mv}^{t+1}\rs\leq&\ls{\mw}^{t}-{\mv}^{t}\rs+\frac{4\beta L}{ (t+1)}\sum_{q=1}^{t-1}\frac{\lambda^{t-q-1}}{q+1}+\frac{2L}{(t+1)m}\mathbb{I}_{[j_t=k]}\nonumber\\
		\leq&\ls{\mw}^{t}-{\mv}^{t}\rs+\frac{4\beta LC_{\lambda}}{(t+1)t}+\frac{2L}{(t+1)m}\mathbb{I}_{[j_t=k]}\label{convex decrease},
	\end{align}
	where the last inequality relies on \cref{sum}.
	
	We can repeatedly use the aforementioned inequality to deduce
	\begin{align*}
		\ls{\mw}^{T+1}-{\mv}^{T+1}\rs\leq&\sum_{t=1}^{T}\lp\frac{4\beta LC_{\lambda}}{(t+1)t}+\frac{2L}{(t+1)m}\mathbb{I}_{[j_t=k]}\rp\\
		=&4\beta LC_{\lambda}\sum_{t=1}^{T}\lp\frac{1}{t}-\frac{1}{t+1}\rp+\frac{2L}{m}\sum_{t=1}^{T}\frac{1}{t+1}\mathbb{I}_{[j_t=k]}\\
		\leq&\frac{4\beta LC_{\lambda}T}{T+1}+\frac{2L}{m}\sum_{t=1}^{T}\frac{1}{t+1}\mathbb{I}_{[j_t=k]}.
	\end{align*}
	Then
	\begin{align*}
		\epsilon_{rk}=&\frac{4\beta L^2C_{\lambda}T}{T+1}+\frac{2L^2}{m}\sum_{t=1}^{T}\frac{1}{t+1}\mathbb{I}_{[j_t=k]}\\
		=&2L^2\lp\frac{2\beta C_{\lambda}T}{T+1}+\frac{1}{m}\sum_{t=1}^{T}\frac{1}{t+1}\mathbb{I}_{[j_t=k]}\rp.
	\end{align*}
	It then follows that
	\begin{align*}
		\frac{1}{mn}\sum_{r=1}^{m}\sum_{k=1}^{n}\epsilon_{rk}^2=\frac{4L^4}{n}\sum_{k=1}^{n}\lp\frac{2\beta C_{\lambda}T}{T+1}+\frac{1}{m}\sum_{t=1}^{T}\frac{1}{t+1}\mathbb{I}_{[j_t=k]}\rp^2
	\end{align*}
	and
	\begin{align*}
		\frac{1}{mn}\sum_{r=1}^{m}\sum_{k=1}^{n}\epsilon_{rk}=&\frac{2L^2}{ mn}\sum_{r=1}^{m}\sum_{k=1}^{n}\lp\frac{2\beta C_{\lambda}T}{T+1}+\frac{1}{m}\sum_{t=1}^{T}\frac{1}{t+1}\mathbb{I}_{[j_t=k]}\rp\\
		=&{2L^2}\lp\frac{2\beta T}{T+1}+\frac{\ln \lp T+1\rp}{mn}\rp.
	\end{align*}
\end{proof}
\subsection{Average weight}\label{Average weight}
It is easy to know that
\begin{align*}
	\ls\bar{\mw}^{T+1}-\bar{\mv}^{T+1}\rs=\ls\frac{\sum_{t=1}^{T+1}\eta_t\lp\mw^{t}-\mv^{t}\rp}{\sum_{t=1}^{T+1}\eta_t}\rs
	\leq\frac{\sum_{t=1}^{T+1}\eta_t\ls\mw^{t}-\mv^{t}\rs}{\sum_{t=1}^{T+1}\eta_t}.
\end{align*}
According to the \cref{convex recursion}, we can get 
\begin{equation*}
	\ls{\mw}^{t}-{\mv}^{t}\rs\leq4\beta L\sum_{{s}=1}^{t-1}\eta_{s}\sum_{q=1}^{{s}-1}\eta_q\lambda^{{s}-q-1}+\frac{2L}{m}\sum_{s=1}^{t-1}\eta_{s}\mathbb{I}_{[j_{s}=k]}.
\end{equation*}
Furthermore, 
\begin{align*}
	\ls\bar{\mw}^{T+1}-\bar{\mv}^{T+1}\rs\leq&\frac{\sum_{t=1}^{T+1}\eta_t\ls\mw^{t}-\mv^{t}\rs}{\sum_{t=1}^{T+1}\eta_t}\\
	\leq&\frac{\sum_{t=1}^{T+1}\eta_t\lp4\beta L\sum_{s=1}^{t-1}\eta_{s}\sum_{q=1}^{s-1}\eta_q\lambda^{s-q-1}+\frac{2L}{m}\sum_{s=1}^{t-1}\eta_{s}\mathbb{I}_{[j_{s}=k]}\rp}{\sum_{t=1}^{T+1}\eta_t}.
\end{align*}
We also consider the two cases (Constant stepsize case and Decreasing stepsize case)
\begin{proof}
	\textbf{Case I [Constant stepsize]:}
	
	If $\eta_t\equiv \eta\leq\frac{2}{L}$,
	by the \cref{convex constant}, we know 
	\begin{align*}
		\ls{\mw}^{t}-{\mv}^{t}\rs\leq
		\frac{4\eta^2\beta L(t-1)}{1-\lambda}+\frac{2\eta L}{m}\sum_{s=1}^{t-1}\mathbb{I}_{[j_{s}=k]}.
	\end{align*}
	Then,
	\begin{align*}
		\ls\bar{\mw}^{T+1}-\bar{\mv}^{T+1}\rs\leq&\frac{\eta\sum_{t=1}^{T+1}\lp\frac{4\eta^2\beta L(t-1)}{1-\lambda}+\frac{2\eta L}{m}\sum_{s=1}^{t-1}\mathbb{I}_{[j_{s}=k]}\rp}{(T+1)\eta}\\
		\leq&\frac{2\eta^2\beta LT}{1-\lambda}+\frac{2\eta L}{(T+1)m}\sum_{t=1}^{T+1}\sum_{s=1}^{t-1}\mathbb{I}_{[j_{s}=k]}.
	\end{align*}
	We can obtain 
	\begin{align*}
		\epsilon_{rk}=&\frac{2\eta^2\beta L^2T}{1-\lambda}+\frac{2\eta L^2}{(T+1)m}\sum_{t=1}^{T}(T+1-t)\mathbb{I}_{[j_t=k]}\\
		=&2L^2\eta\lp\frac{\eta\beta T}{1-\lambda}+\frac{\eta}{(T+1)m}\sum_{t=1}^{T}(T+1-t)\mathbb{I}_{[j_t=k]}\rp.
	\end{align*}
	It then follows that
	\begin{align*}
		\frac{1}{mn}\sum_{r=1}^{m}\sum_{k=1}^{n}\epsilon_{rk}^2=\frac{4L^4\eta^2}{n}\sum_{k=1}^{n}\lp\frac{\eta\beta T}{1-\lambda}+\frac{1}{(T+1)m}\sum_{t=1}^{T}(T+1-t)\mathbb{I}_{[j_t=k]}\rp^2
	\end{align*}
	and
	\begin{align*}
		\frac{1}{mn}\sum_{r=1}^{m}\sum_{k=1}^{n}\epsilon_{rk}=&\frac{2L^2\eta}{mn}\sum_{r=1}^{m}\sum_{k=1}^{n}\lp\frac{\eta\beta T}{1-\lambda}+\frac{1}{(T+1)m}\sum_{t=1}^{T}(T+1-t)\mathbb{I}_{[j_t=k]}\rp\\
		=&{2L^2\eta T}\lp\frac{\eta\beta}{1-\lambda}+\frac{1}{mn}\rp.
	\end{align*}
	\textbf{Case II [Decreasing stepsize]:}
	
	If $\eta_t=\frac{1}{t+1}$, by the \cref{convex decrease}, we can get
	\begin{align*}
		\ls{\mw}^{t+1}-{\mv}^{t+1}\rs
		\leq\ls{\mw}^{t}-{\mv}^{t}\rs+\frac{4\beta LC_{\lambda}}{(t+1)t}+\frac{2L}{(t+1)m}\mathbb{I}_{[j_t=k]},
	\end{align*}
	Then, 
	\begin{align*}
		\ls\bar{\mw}^{T+1}-\bar{\mv}^{T+1}\rs\leq&\frac{\sum_{t=1}^{T+1}\frac{1}{t+1}\lp\frac{4\beta LC_{\lambda}(t-1)}{t}+\frac{2L }{m}\sum_{s=1}^{t-1}\frac{1}{s+1}\mathbb{I}_{[j_{s}=k]}\rp}{\sum_{t=1}^{T+1}\frac{1}{t+1}}\\
		\leq&\frac{4\beta LC_{\lambda}\sum_{t=1}^{T+1}\frac{t-1}{(t+1)t}+\frac{2L}{m}\sum_{t=1}^{T+1}\frac{1}{t+1}\sum_{s=1}^{t-1}\frac{1}{s+1}\mathbb{I}_{[j_{s}=k]}}{\ln (T+2)}\\
		\leq&4\beta LC_{\lambda}+\frac{2L}{m\ln (T+2)}\sum_{t=1}^{T+1}\frac{1}{t+1}\sum_{s=1}^{t-1}\frac{1}{s+1}\mathbb{I}_{[j_{s}=k]},
	\end{align*}
	where the last inequality uses
	\begin{align*}
		\sum_{t=1}^{T+1}\frac{t-1}{(t+1)t}=\sum_{t=1}^{T+1}\lp\frac{2}{t+1}-\frac{1}{t}\rp
		\leq2\ln(T+2)-\ln(T+1).
	\end{align*}
	It then follows that
	\begin{align*}
		\frac{1}{mn}\sum_{r=1}^{m}\sum_{k=1}^{n}\epsilon_k^2=\frac{4L^4}{n}\sum_{k=1}^{n}\lp2\beta C_{\lambda}+\frac{2L}{m\ln (T+2)}\sum_{t=1}^{T+1}\frac{1}{t+1}\sum_{s=1}^{t-1}\frac{1}{s+1}\mathbb{I}_{[j_{s}=k]}\rp^2
	\end{align*}
	and
	\begin{align*}
		\frac{1}{mn}\sum_{r=1}^{m}\sum_{k=1}^{n}\epsilon_k=&\frac{2L^2}{mn}\sum_{r=1}^{m}\sum_{k=1}^{n}\lp2\beta C_{\lambda}+\frac{1}{m\ln (T+2)}\sum_{t=1}^{T+1}\frac{1}{t+1}\sum_{s=1}^{t-1}\frac{1}{s+1}\mathbb{I}_{[j_{s}=k]}\rp\\
		=&{4L^2}\beta C_{\lambda}+\frac{2L^2}{mn\ln (T+2)}\sum_{t=1}^{T+1}\frac{1}{t+1}\sum_{s=1}^{t-1}\frac{1}{s+1}\\
		\leq&4L^2\beta C_{\lambda}+\frac{2L^2}{mn\ln (T+2)}\sum_{t=1}^{T+1}\frac{\ln t}{t+1}\\
		\leq&4L^2\beta C_{\lambda}+\frac{L^2\ln (T+2)}{mn},
	\end{align*}
	where the last inequality relies on 
	\begin{align*}
		\sum_{t=1}^{T+1}\frac{\ln t}{t+1}\leq\frac{\ln^2(T+2)}{2}.
	\end{align*}
\end{proof}

\section{Strongly Convex Case}\label{Section strongly convex}
\subsection{Proof of \cref{strong convex theorem 1}}
\begin{proof}
	With the probability $1-\frac{1}{n}$, if $j_t\neq k$, 
	\begin{align*}
		&\ls{\mw}^{t+1}-{\mv}^{t+1}\rs\\
		=&\ls\frac{1}{m}\sum_{i=1}^{m}\mw^{t+1}(i)-\frac{1}{m}\sum_{i=1}^{m}\mv^{t+1}(i)\rs\\
		=&\frac{1}{m}\sum_{i=1}^{m}\ls\sum_{l=1}^{m}P_{i,l}\mw^{t}(l)-\eta_t\nabla f\lp\mw^{t}(i);Z_{j_t(i)}\rp-\sum_{l=1}^{m}P_{i,l}\mv^{t}(l)+\eta_t\nabla f\lp\mv^{t}(i);{Z}_{j_t(i)}\rp\rs\\
		\leq&\frac{1}{m}\ls\sum_{i=1}^{m}\sum_{l=1}^{m}P_{i,l}\mw^{t}(l)-\sum_{i=1}^{m}\eta_t\nabla f\lp\mw^{t};Z_{j_t(i)}\rp-\sum_{i=1}^{m}\sum_{l=1}^{m}P_{i,l}\mv^{t}(l)+\sum_{i=1}^{m}\eta_t\nabla f\lp\mv^{t};Z_{j_t(i)}\rp\rs\\
		&\quad\quad+\frac{\eta_t\beta}{m}\sum_{i=1}^{m}\ls\mw^{t}-\mw^{t}(i)\rs+\frac{\eta_t\beta}{m}\sum_{i=1}^{m}\ls\mv^{t}-\mv^{t}(i)\rs\\
		\leq&\ls\mw^{t}-\frac{1}{m}\sum_{i=1}^{m}\eta_t\nabla f\lp\mw^{t};Z_{j_t(i)}\rp-\mv^{t}+\frac{1}{m}\sum_{i=1}^{m}\eta_t\nabla f\lp\mv^{t};Z_{j_t(i)}\rp\rs\\
		&\quad\quad+\frac{\eta_t\beta}{\sqrt{m}}\lk\sum_{i=1}^{m}\ls{\mw}^{t}-\mw^{t}(i)\rs^2\rk^\frac{1}{2}+\frac{\eta_t\beta}{\sqrt{m}}\lk\sum_{i=1}^{m}\ls{\mv}^{t}-\mv^{t}(i)\rs^2\rk^\frac{1}{2}\\
		\leq&\lp1-\frac{\eta_t\mu}{2}\rp\ls{\mw}^{t}-{\mv}^{t}\rs+4\eta_t\beta L\sum_{q=1}^{t-1}\eta_q\lambda^{t-q-1},
	\end{align*}
	where the last inequality uses the Lemma \ref{non}.
	
	With the probability $\frac{1}{n}$, if $j_t=k$, 
	\begin{align*}
		&\ls{\mw}^{t+1}-{\mv}^{t+1}\rs\\
		=&\ls\frac{1}{m}\sum_{i=1}^{m}\mw^{t+1}(i)-\frac{1}{m}\sum_{i=1}^{m}\mv^{t+1}(i)\rs\\
		\leq&\frac{1}{m}\sum_{i=1}^{m}\ls\sum_{l=1}^{m}P_{i,l}\mw^{t}(l)-\eta_t\nabla f\lp\mw^{t};Z_{k(i)}\rp-\sum_{l=1}^{m}P_{i,l}\mv^{t}(l)+\eta_t\nabla f\lp\mv^{t};\tilde{Z}_{k(i)}\rp\rs
		+4\eta_t\beta L\sum_{q=1}^{t-1}\eta_q\lambda^{t-q-1}\\
		\leq&\ls\mw^{t}-\frac{\eta_t}{m}\sum_{i\neq r}^{m}\nabla f\lp\mw^{t};Z_{k(i)}\rp-\mv^t+\frac{\eta_t}{m}\sum_{i\neq r}^{m}\nabla f\lp\mv^{t};Z_{k(i)}\rp\rs\\
		&+\frac{\eta_t}{m}\ls\nabla f\lp\mw^{t};Z_{k(r)}\rp-\nabla f\lp\mv^{t};\tilde{Z}_{k(r)}\rp\rs
		+4\eta_t\beta L\sum_{q=1}^{t-1}\eta_q\lambda^{t-q-1}\\
		\leq&\lp1-\frac{\eta_t\mu}{2}\rp\ls{\mw}^{t}-{\mv}^{t}\rs+4\eta_t\beta L\sum_{q=1}^{t-1}\eta_q\lambda^{t-q-1}+\frac{2\eta_tL}{m}.
	\end{align*}
	We combine the above two cases, and get
	\begin{equation}\label{strong re}
		\ls{\mw}^{t+1}-{\mv}^{t+1}\rs\leq\lp1-\frac{\eta_t\mu}{2}\rp\ls{\mw}^{t}-{\mv}^{t}\rs+4\eta_t\beta L\sum_{q=1}^{t-1}\eta_q\lambda^{t-q-1}+\frac{2\eta_tL}{m}\mathbb{I}_{[j_t=k]}.
	\end{equation}
	Apply the above inequality recursively and obtain
	\begin{equation}
		\ls{\mw}^{T+1}-{\mv}^{T+1}\rs\leq4\beta L\sum_{t=1}^{T}\eta_t\sum_{q=1}^{t-1}\eta_q\lambda^{t-q-1}\prod_{\tilde{t}=t+1}^{T}\lp1-\frac{\eta_{\tilde{t}}\mu}{2}\rp+\frac{2L}{m}\sum_{t=1}^{T}\eta_t\mathbb{I}_{[j_t=k]}\prod_{\tilde{t}=t+1}^{T}\lp1-\frac{\eta_{\tilde{t}}\mu}{2}\rp.
	\end{equation}
	Therefore, 
	\begin{align*}
		\epsilon_{rk}=2L^2\sum_{t=1}^{T}\lp2\beta\eta_t\sum_{q=1}^{t-1}\eta_q\lambda^{t-q-1}+\frac{\eta_t}{m}\mathbb{I}_{[j_t=k]}\rp\prod_{\tilde{t}=t+1}^{T}\lp1-\frac{\eta_{\tilde{t}}\mu}{2}\rp.
	\end{align*}
	It then follows that 
	\begin{align*}
		\frac{1}{mn}\sum_{r=1}^{m}\sum_{k=1}^{n}\epsilon_{rk}^2=\frac{4L^4}{n}\sum_{k=1}^{n}\lp\sum_{t=1}^{T}\lp2\beta\eta_t\sum_{q=1}^{t-1}\eta_q\lambda^{t-q-1}+\frac{\eta_t}{m}\mathbb{I}_{[j_t=k]}\rp\prod_{\tilde{t}=t+1}^{T}\lp1-\frac{\eta_{\tilde{t}}\mu}{2}\rp\rp^2
	\end{align*}
	and 
	\begin{align*}
		\frac{1}{mn}\sum_{r=1}^{m}\sum_{k=1}^{n}\epsilon_k=&\frac{2L^2}{mn}\sum_{r=1}^{m}\sum_{k=1}^{n}\lp\sum_{t=1}^{T}\lp2\beta\eta_t\sum_{q=1}^{t-1}\eta_q\lambda^{t-q-1}+\frac{\eta_t}{m}\mathbb{I}_{[j_t=k]}\rp\prod_{\tilde{t}=t+1}^{T}\lp1-\frac{\eta_{\tilde{t}}\mu}{2}\rp\rp\\
		{=}&4\beta L^2\sum_{t=1}^{T}\eta_t\sum_{q=1}^{t-1}\eta_q\lambda^{t-q-1}\prod_{\tilde{t}=t+1}^{T}\lp1-\frac{\eta_{\tilde{t}}\mu}{2}\rp+\frac{2L^2}{mn}\sum_{t=1}^{T}\eta_t\prod_{\tilde{t}=t+1}^{T}\lp1-\frac{\eta_{\tilde{t}}\mu}{2}\rp\\
		\overset{\textbf{(1)}}{=}&4\beta L^2\sum_{t=1}^{T}\eta_t\sum_{q=1}^{t-1}\eta_q\lambda^{t-q-1}\prod_{\tilde{t}=t+1}^{T}\lp1-\frac{\eta_{\tilde{t}}\mu}{2}\rp+\frac{4L^2}{mn\mu}\lp1-\prod_{t=1}^{T}\lp1-\frac{\eta_{\tilde{t}}\mu}{2}\rp\rp\\
		\leq&4\beta L^2\sum_{t=1}^{T}\eta_t\sum_{q=1}^{t-1}\eta_q\lambda^{t-q-1}\prod_{\tilde{t}=t+1}^{T}\lp1-\frac{\eta_{\tilde{t}}\mu}{2}\rp+\frac{4L^2}{mn\mu}
		,
	\end{align*}
	where the Equation $\textbf{(1)}$ uses
	\begin{align*}
		\sum_{t=1}^{T}\lp\prod_{\tilde{t}=t+1}^{T}\lp1-\frac{\eta_{\tilde{t}}\mu}{2}\rp-\prod_{\tilde{t}=t}^{T}\lp1-\frac{\eta_{\tilde{t}}\mu}{2}\rp\rp=\sum_{t=1}^{T}\lp1-\lp1-\frac{\eta_t\mu}{2}\rp\rp\prod_{\tilde{t}=t+1}^{T}\lp1-\frac{\eta_{\tilde{t}}\mu}{2}\rp.
	\end{align*}
\end{proof}
\subsection{Discussion about learning rate}\label{compare 11}
\begin{proof}
	\textbf{Case I [Constant stepsize]:}
	
	When $\eta_t\equiv \eta\leq\frac{1}{L}$, then we can get
	\begin{equation}
		\ls{\mw}^{t+1}-{\mv}^{t+1}\rs\leq\lp1-\frac{\eta\mu}{2}\rp\ls{\mw}^{t}-{\mv}^{t}\rs+\frac{4\eta^2\beta L}{1-\lambda}+\frac{2\eta L}{m}\mathbb{I}_{[j_t=k]}.
	\end{equation}
	Recursively iterate on the aforementioned inequality,
	\begin{align*}
		\ls{\mw}^{T+1}-{\mv}^{T+1}\rs\leq&\sum_{t=0}^{T-1}\lp1-\frac{\eta\mu}{2}\rp^t\frac{4L\eta^2\beta}{1-\lambda}+\frac{2L\eta}{m}\sum_{t=1}^{T}\lp1-\frac{\eta\mu}{2}\rp^{T-t}\mathbb{I}_{[j_t=k]}\\
		\leq&\frac{8L\eta\beta }{\mu(1-\lambda)}+\frac{2L\eta}{m}\sum_{t=1}^{T}\lp1-\frac{\eta\mu}{2}\rp^{T-t}\mathbb{I}_{[j_{t}=k]}.		
	\end{align*}
	Therefore, 
	\begin{align*}
		\epsilon_{rk}=\frac{8L^2\eta\beta }{\mu(1-\lambda)}+\frac{2L^2\eta}{m}\sum_{t=1}^{T}\lp1-\frac{\eta\mu}{2}\rp^{T-t}\mathbb{I}_{[j_{t}=k]}.
	\end{align*}
	It then follows that 
	\begin{align*}
		\frac{1}{mn}\sum_{r=1}^{m}\sum_{k=1}^{n}\epsilon_{rk}^2=\frac{4L^4\eta^2}{n}\sum_{k=1}^{n}\lp\frac{4\beta}{\mu(1-\lambda)}+\frac{1}{m}\sum_{t=1}^{T}\lp1-\frac{\eta\mu}{2}\rp^{T-t}\mathbb{I}_{[j_{t}=k]}\rp^2
	\end{align*}
	and
	\begin{align*}
		\frac{1}{mn}\sum_{r=1}^{m}\sum_{k=1}^{n}\epsilon_{rk}=&\frac{2L^2\eta}{ mn}\sum_{r=1}^{m}\sum_{k=1}^{n}\lp\frac{4\beta}{\mu(1-\lambda)}+\frac{1}{m}\sum_{t=1}^{T}\lp1-\frac{\eta\mu}{2}\rp^{T-t}\mathbb{I}_{[j_{t}=k]}\rp\\
		=&\frac{8L^2\eta\beta}{\mu(1-\lambda)}+\frac{2L^2\eta}{mn}\sum_{t=1}^{T}\lp1-\frac{\eta\mu}{2}\rp^{T-t}\\
		=&\frac{4L^2}{\mu }\lp\frac{2\eta\beta}{1-\lambda}+\frac{1}{mn}\rp.
	\end{align*}
	Moreover, it is easy to know that 
	\begin{align*}
		\epsilon_{\textrm{uniform}}=\frac{8L^2\eta\beta}{\mu(1-\lambda)}+\frac{2L^2\eta}{m}\max_{k\in[n]}\sum_{t=1}^{T}\lp1-\frac{\eta\mu}{2}\rp^{T-t}\mathbb{I}_{[j_{t}=k]}.
	\end{align*}
	Then it also straightforward to verify that $\lk\varDelta_{rk}^2\rk^{1/2}\leq\epsilon_{\textrm{uniform}}$.\\
	\textbf{Case II [Decreasing stepsize]:}
	
	When $\eta_t=\frac{2}{\mu(t+1)}$, then we can get
	\begin{align*}
		\ls{\mw}^{t+1}-{\mv}^{t+1}\rs\leq&\lp1-\frac{1}{t+1}\rp\ls{\mw}^{t}-{\mv}^{t}\rs+\frac{16\beta L}{\mu^2 (t+1)}\sum_{q=1}^{t-1}\frac{\lambda^{t-q-1}}{q+1}+\frac{4L}{\mu (t+1)m}\mathbb{I}_{[j_t=k]}\\
		\leq&\lp1-\frac{1}{t+1}\rp\ls{\mw}^{t}-{\mv}^{t}\rs+\frac{16\beta LC_{\lambda}}{\mu^2 (t+1)t}+\frac{4L}{\mu (t+1)m}\mathbb{I}_{[j_t=k]},
	\end{align*}
	where the last inequality rely on the Lemma \ref{sum}.
	
	Apply recursive iteration to the previously mentioned inequality,
	\begin{align*}
		\ls{\mw}^{T+1}-{\mv}^{T+1}\rs\leq&4L\sum_{t=1}^{T}\prod_{\tilde{t}=t+1}^{T}\frac{\tilde{t}}{\tilde{t}+1}\lp\frac{4\beta C_{\lambda}}{\mu^2 (t+1)t}+\frac{1}{\mu (t+1)m}\mathbb{I}_{[j_t(r)=k]}\rp\\
		\leq&4L\sum_{t=1}^{T}\frac{t+1}{T}\lp\frac{4\beta C_{\lambda}}{\mu^2 (t+1)t}+\frac{1}{\mu (t+1)m}\mathbb{I}_{[j_t=k]}\rp\\
		\leq&\frac{4L}{T}\sum_{t=1}^{T}\lp\frac{4\beta C_{\lambda}}{\mu^2t}+\frac{1}{\mu m}\mathbb{I}_{[j_t(r)=k]}\rp\\
		\leq&\frac{16\beta LC_{\lambda}}{T\mu^2}\lp\ln T+1\rp+\frac{4L}{T\mu m}\sum_{t=1}^{T}\mathbb{I}_{[j_t=k]},
	\end{align*}
	where the last inequality uses $\sum_{t=1}^{T}\frac{1}{t}\leq \ln T+1$.
	
	Therefore, 
	\begin{align*}
		\epsilon_{rk}=&\frac{16\beta L^2C_{\lambda}}{T\mu^2}\lp\ln T+1\rp+\frac{4L^2}{T\mu m}\sum_{t=1}^{T}\mathbb{I}_{[j_t=k]}\\
		=&\frac{4L^2}{T\mu}\lp\frac{4\beta C_{\lambda}}{\mu}\lp\ln T+1\rp+\frac{1}{m}\sum_{t=1}^{T}\mathbb{I}_{[j_t=k]}\rp.
	\end{align*}
	It then follows that 
	\begin{align*}
		\frac{1}{mn}\sum_{r=1}^{m}\sum_{k=1}^{n}\epsilon_{rk}^2=\frac{16L^4}{T^2\mu n}\sum_{k=1}^{n}\lp\frac{4\beta C_{\lambda}}{\mu}\lp\ln T+1\rp+\frac{1}{m}\sum_{t=1}^{T}\mathbb{I}_{[j_t=k]}\rp^2
	\end{align*}
	and
	\begin{align*}
		\frac{1}{mn}\sum_{r=1}^{m}\sum_{k=1}^{n}\epsilon_{rk}=&\frac{4L^2}{T\mu mn}\sum_{r=1}^{m}\sum_{k=1}^{n}\lp\frac{4\beta C_{\lambda}}{\mu}\lp\ln T+1\rp+\frac{1}{m}\sum_{t=1}^{T}\mathbb{I}_{[j_t=k]}\rp\\
		=&\frac{16\beta L^2C_{\lambda}}{T\mu^2}\lp\ln T+1\rp+\frac{4L^2}{\mu mn}.
	\end{align*}
\end{proof}

\section{Nonconvex Case}\label{Section Nonconvex}
\subsection{Proof of \cref{nonconvex case theorem 1}}

With the probability $1-\frac{1}{n}$, if $j_t\neq k$, 
\begin{align*}
	&\ls{\mw}^{t+1}-{\mv}^{t+1}\rs\\
	=&\ls\frac{1}{m}\sum_{i=1}^{m}\mw^{t+1}(i)-\frac{1}{m}\sum_{i=1}^{m}\mv^{t+1}(i)\rs\\
	=&\frac{1}{m}\sum_{i=1}^{m}\ls\sum_{l=1}^{m}P_{i,l}\mw^{t}(l)-\eta_t\nabla f\lp\mw^{t}(i);Z_{j_t(i)}\rp-\sum_{l=1}^{m}P_{i,l}\mv^{t}(l)+\eta_t\nabla f\lp\mv^{t}(i);{Z}_{j_t(i)}\rp\rs\\
	\leq&\frac{1}{m}\ls\sum_{i=1}^{m}\sum_{l=1}^{m}P_{i,l}\mw^{t}(l)-\sum_{i=1}^{m}\eta_t\nabla f\lp\mw^{t};Z_{j_t(i)}\rp-\sum_{i=1}^{m}\sum_{l=1}^{m}P_{i,l}\mv^{t}(l)+\sum_{i=1}^{m}\eta_t\nabla f\lp\mv^{t};Z_{j_t(i)}\rp\rs\\
	&\quad\quad+\frac{\eta_t\beta}{m}\sum_{i=1}^{m}\ls\mw^{t}-\mw^{t}(i)\rs+\frac{\eta_t\beta}{m}\sum_{i=1}^{m}\ls\mv^{t}-\mv^{t}(i)\rs\\
	\leq&\ls\mw^{t}-\frac{1}{m}\sum_{i=1}^{m}\eta_t\nabla f\lp\mw^{t};Z_{j_t(i)}\rp-\mv^{t}+\frac{1}{m}\sum_{i=1}^{m}\eta_t\nabla f\lp\mv^{t};Z_{j_t(i)}\rp\rs\\
	&\quad\quad+\frac{\eta_t\beta}{\sqrt{m}}\lk\sum_{i=1}^{m}\ls{\mw}^{t}-\mw^{t}(i)\rs^2\rk^\frac{1}{2}+\frac{\eta_t\beta}{\sqrt{m}}\lk\sum_{i=1}^{m}\ls{\mv}^{t}-\mv^{t}(i)\rs^2\rk^\frac{1}{2}\\
	\leq&\ls\mw^{t}-\mv^{t}\rs+\frac{1}{m}\ls\sum_{i=1}^{m}\eta_t\nabla f\lp\mw^{t};Z_{j_t(i)}\rp-\sum_{i=1}^{m}\eta_t\nabla f\lp\mv^{t};Z_{j_t(i)}\rp\rs\\
	&\quad\quad+\frac{\eta_t\beta}{\sqrt{m}}\lk\sum_{i=1}^{m}\ls{\mw}^{t}-\mw^{t}(i)\rs^2\rk^\frac{1}{2}+\frac{\eta_t\beta}{\sqrt{m}}\lk\sum_{i=1}^{m}\ls{\mv}^{t}-\mv^{t}(i)\rs^2\rk^\frac{1}{2}\\
	\leq&\lp1+\beta\eta_t\rp\ls{\mw}^{t}-{\mv}^{t}\rs+4\eta_t\beta L\sum_{q=1}^{t-1}\eta_q\lambda^{t-q-1},
\end{align*}
where the last inequality uses the $\beta$-smooth property.

With the probability $\frac{1}{n}$, if $j_t=k$, 
\begin{align*}
	&\ls{\mw}^{t+1}-{\mv}^{t+1}\rs\\
	=&\ls\frac{1}{m}\sum_{i=1}^{m}\mw^{t+1}(i)-\frac{1}{m}\sum_{i=1}^{m}\mv^{t+1}(i)\rs\\
	\leq&\frac{1}{m}\sum_{i=1}^{m}\ls\sum_{l=1}^{m}P_{i,l}\mw^{t}(l)-\eta_t\nabla f\lp\mw^{t};Z_{k(i)}\rp-\sum_{l=1}^{m}P_{i,l}\mv^{t}(l)+\eta_t\nabla f\lp\mv^{t};\tilde{Z}_{k(i)}\rp\rs
	+4\eta_t\beta L\sum_{q=1}^{t-1}\eta_q\lambda^{t-q-1}\\
	\leq&\ls\mw^{t}-\frac{\eta_t}{m}\sum_{i\neq r}^{m}\nabla f\lp\mw^{t};Z_{k(i)}\rp-\mv^t+\frac{\eta_t}{m}\sum_{i\neq r}^{m}\nabla f\lp\mv^{t};Z_{k(i)}\rp\rs+\frac{\eta_t}{m}\ls\nabla f\lp\mw^{t};Z_{k(r)}\rp-\nabla f\lp\mv^{t};\tilde{Z}_{k(r)}\rp\rs\\
	&+4\eta_t\beta L\sum_{q=1}^{t-1}\eta_q\lambda^{t-q-1}\\
	\leq&\lp1+\beta\eta_t\rp\ls{\mw}^{t}-{\mv}^{t}\rs+4\eta_t\beta L\sum_{j=1}^{t-1}\eta_q\lambda^{t-q-1}+\frac{2\eta_tL}{m}.
\end{align*}

We combine the above two cases and obtain
\begin{equation}\label{nonconvex re}
	\ls{\mw}^{t+1}-{\mv}^{t+1}\rs\leq\lp1+\beta\eta_t\rp\ls{\mw}^{t}-{\mv}^{t}\rs+4\eta_t\beta L\sum_{q=1}^{t-1}\eta_q\lambda^{t-q-1}+\frac{2\eta_tL}{m}\mathbb{I}_{[j_t=k]}.
\end{equation}
Apply the \cref{nonconvex re} recursively and get
\begin{equation}\label{nonconvex recursion}
	\ls{\mw}^{T+1}-{\mv}^{T+1}\rs\leq4\beta L\sum_{t=1}^{T}\eta_t\sum_{q=1}^{t-1}\eta_q\lambda^{t-q-1}\prod_{\tilde{t}=t+1}^{T}\lp1+\beta\eta_{\tilde{t}}\rp+\frac{2L}{m}\sum_{t=1}^{T}\eta_t\mathbb{I}_{[j_t=k]}\prod_{\tilde{t}=t+1}^{T}\lp1+\beta\eta_{\tilde{t}}\rp.
\end{equation}
Therefore, 
\begin{align*}
	\epsilon_k=2L^2\sum_{t=1}^{T}\lp2\beta\eta_t\sum_{q=1}^{t-1}\eta_q\lambda^{t-q-1}+\frac{\eta_t}{m}\mathbb{I}_{[j_t=k]}\rp\prod_{\tilde{t}=t+1}^{T}\lp1+\beta\eta_{\tilde{t}}\rp.
\end{align*}
It then follows that 
\begin{align*}
	\frac{1}{mn}\sum_{r=1}^{m}\sum_{k=1}^{n}\epsilon_{rk}^2=\frac{4L^4}{n}\sum_{k=1}^{n}\lp\sum_{t=1}^{T}\lp2\beta\eta_t\sum_{q=1}^{t-1}\eta_q\lambda^{t-q-1}+\frac{\eta_t}{m}\mathbb{I}_{[j_t=k]}\rp\prod_{\tilde{t}=t+1}^{T}\lp1+\beta\eta_{\tilde{t}}\rp\rp^2
\end{align*}

and 
\begin{align*}
	\frac{1}{mn}\sum_{r=1}^{m}\sum_{k=1}^{n}\epsilon_{rk}=&\frac{2L^2}{mn}\sum_{r=1}^{m}\sum_{k=1}^{n}\lp\sum_{t=1}^{T}\lp2\beta\eta_t\sum_{q=1}^{t-1}\eta_q\lambda^{t-q-1}+\frac{\eta_t}{m}\mathbb{I}_{[j_t=k]}\rp\prod_{\tilde{t}=t+1}^{T}\lp1+\beta\eta_{\tilde{t}}\rp\rp\\
	=&4\beta L^2\sum_{t=1}^{T}\eta_t\sum_{q=1}^{t-1}\eta_q\lambda^{t-q-1}\prod_{\tilde{t}=t+1}^{T}\lp1+\beta\eta_{\tilde{t}}\rp+\frac{2L^2}{mn}\sum_{t=1}^{T}\eta_t\prod_{\tilde{t}=t+1}^{T}\lp1+\beta\eta_{\tilde{t}}\rp.
\end{align*}

\subsection{Discussion about learning rate}\label{nonconvex differ learning}
\textbf{Case I [Constant stepsize]:}

When $\eta_t\equiv \eta$, according to \cref{nonconvex re}, we can get
\begin{equation}
	\ls{\mw}^{t+1}-{\mv}^{t+1}\rs\leq\lp1+\beta\eta\rp\ls{\mw}^{t}-{\mv}^{t}\rs+\frac{4\eta^2\beta L}{1-\lambda}+\frac{2\eta_tL}{m}\mathbb{I}_{[j_t=k]}.
\end{equation}
Recursively iterate on the aforementioned inequality,
\begin{align}
	\ls{\mw}^{T+1}-{\mv}^{T+1}\rs
	\leq&\sum_{t=0}^{T-1}\lp1+\beta\eta\rp^t\frac{4L\eta^2\beta }{1-\lambda}+\frac{2L\eta}{m}\sum_{t=1}^{T}\lp1+\beta\eta\rp^{T-t}\mathbb{I}_{[j_t(r)=k]}\nonumber\\
	\leq&\frac{4L\eta}{1-\lambda}\lp1+\beta\eta\rp^T+\frac{2L\eta}{ m}\sum_{t=1}^{T}\lp1+\beta\eta\rp^{T-t}\mathbb{I}_{[j_{t}=k]}\label{nonconvex constant}.		
\end{align}
Therefore, 
\begin{align*}
	\epsilon_{rk}=\frac{4L^2\eta}{1-\lambda}\lp1+\beta\eta\rp^T+\frac{2L^2\eta}{ m}\sum_{t=1}^{T}\lp1+\beta\eta\rp^{T-t}\mathbb{I}_{[j_{t}=k]}.
\end{align*}
It then follows that 
\begin{align*}
	\frac{1}{mn}\sum_{r=1}^{m}\sum_{k=1}^{n}\epsilon_{rk}^2=\frac{4L^4\eta^2}{n}\sum_{k=1}^{n}\lp\frac{2}{1-\lambda}\lp1+\beta\eta\rp^T+\frac{1}{ m}\sum_{t=1}^{T}\lp1+\beta\eta\rp^{T-t}\mathbb{I}_{[j_{t}=k]}\rp^2,
\end{align*}
and
\begin{align*}
	\frac{1}{mn}\sum_{r=1}^{m}\sum_{k=1}^{n}\epsilon_{rk}=&\frac{2L^2\eta}{ mn}\sum_{r=1}^{m}\sum_{k=1}^{n}\lp\frac{2}{1-\lambda}\lp1+\beta\eta\rp^T+\frac{1}{ m}\sum_{t=1}^{T}\lp1+\beta\eta\rp^{T-t}\mathbb{I}_{[j_{t}=k]}\rp\\
	=&\frac{4L^2\eta}{1-\lambda}\lp1+\beta\eta\rp^T+\frac{2L^2\eta}{mn}\sum_{t=1}^{T}\lp1+\beta\eta\rp^{T-t}\\
	=&{2L^2}\lp\frac{2\eta}{1-\lambda}+\frac{1}{mn\beta}\rp\lp1+\beta\eta\rp^T.
\end{align*}

\noindent\textbf{Case II [Decreasing stepsize]:}

When $\eta_t=\frac{1}{t+1}$, then we have
\begin{align*}
	\ls{\mw}^{t+1}-{\mv}^{t+1}\rs\leq&\lp1+\frac{\beta}{t+1}\rp\ls{\mw}^{t}-{\mv}^{t}\rs+\frac{4\beta L}{(t+1)}\sum_{q=1}^{t-1}\frac{\lambda^{t-q-1}}{q+1}+\frac{2L}{(t+1)m}\mathbb{I}_{[j_t=k]}\\
	\leq&\lp1+\frac{\beta}{t+1}\rp\ls{\mw}^{t}-{\mv}^{t}\rs+\frac{4\beta LC_{\lambda}}{ (t+1)t}+\frac{2L}{(t+1)m}\mathbb{I}_{[j_t=k]}.
\end{align*}

Apply recursive iteration to the above inequality,
\begin{align*}
	\ls{\mw}^{T+1}-{\mv}^{T+1}\rs\leq&2L\sum_{t=1}^{T}\prod_{\tilde{t}=t+1}^{T}\lp1+\frac{\beta}{\tilde{t}+1}\rp\lp\frac{4\beta C_{\lambda}}{(t+1)t}+\frac{1}{(t+1)m}\mathbb{I}_{[j_t=k]}\rp\\
	\leq&4L\sum_{t=1}^{T}\prod_{\tilde{t}=t+1}^{T}\exp\lp\frac{\beta}{\tilde{t}+1}\rp\lp\frac{4\beta C_{\lambda}}{(t+1)t}+\frac{1}{(t+1)m}\mathbb{I}_{[j_t=k]}\rp\\
	\leq&4L\sum_{t=1}^{T}\exp\lp\beta\sum_{\tilde{t}=t+1}^{T}\frac{1}{\tilde{t}+1}\rp\lp\frac{4\beta C_{\lambda}}{(t+1)t}+\frac{1}{(t+1)m}\mathbb{I}_{[j_t=k]}\rp\\
	\leq&4L\sum_{t=1}^{T}\exp\lp\beta\log\frac{T+1}{t+1}\rp\lp\frac{4\beta C_{\lambda}}{(t+1)t}+\frac{1}{(t+1)m}\mathbb{I}_{[j_t=k]}\rp\\
	\leq&4L(T+1)^\beta\sum_{t=1}^{T}\lp\frac{1}{t+1}\rp^\beta\lp\frac{4\beta C_{\lambda}}{(t+1)t}+\frac{1}{(t+1)m}\mathbb{I}_{[j_t=k]}\rp\\
	\leq&4L(T+1)^\beta\sum_{t=1}^{T}\lp\frac{4\beta C_{\lambda}}{t(t+1)^{\beta+1}}+\frac{1}{m(t+1)^{\beta+1}}\mathbb{I}_{[j_t=k]}\rp\\
	\leq&4L(T+1)^\beta\sum_{t=1}^{T}\lp\frac{4\beta C_{\lambda}}{t^{\beta+2}}+\frac{1}{m(t+1)^{\beta+1}}\mathbb{I}_{[j_t=k]}\rp\\
	\leq&4L(T+1)^{\beta}\lp4 C_{\lambda}+\sum_{t=1}^{T}\frac{1}{m(t+1)^{\beta+1}}\mathbb{I}_{[j_t=k]}\rp,
\end{align*}
where the last inequality uses $\sum_{t=1}^{T}\frac{1}{t^{\beta+2}}\leq \frac{1}{\beta+1}\lp1-\frac{1}{T^{\beta+1}}\rp+1$.

Therefore, 
\begin{align*}
	\epsilon_{rk}=4L^2(T+1)^{\beta}\lp4 C_{\lambda}+\sum_{t=1}^{T}\frac{1}{m(t+1)^{\beta+1}}\mathbb{I}_{[j_t=k]}\rp.
\end{align*}
It then follows that 
\begin{align*}
	\frac{1}{mn}\sum_{r=1}^{m}\sum_{k=1}^{n}\epsilon_{rk}^2=\frac{16L^4}{n}(T+1)^{2\beta}\sum_{k=1}^{n}\lp4 C_{\lambda}+\sum_{t=1}^{T}\frac{1}{m(t+1)^{\beta+1}}\mathbb{I}_{[j_t=k]}\rp^2
\end{align*}
and
\begin{align*}
	\frac{1}{mn}\sum_{r=1}^{m}\sum_{k=1}^{n}\epsilon_{rk}=&\frac{4L^2}{mn}(T+1)^{\beta}\sum_{r=1}^{m}\sum_{k=1}^{n}\lp4 C_{\lambda}+\sum_{t=1}^{T}\frac{1}{m(t+1)^{\beta+1}}\mathbb{I}_{[j_t=k]}\rp\\
	=&4L^2(T+1)^{\beta}\lp{4 C_{\lambda}}+\frac{1}{mn}\sum_{t=1}^{T}\frac{1}{(t+1)^{\beta+1}}\rp\\
	\leq&4L^2(T+1)^{\beta}\lp{4 C_{\lambda}}+\frac{1}{\beta mn}\rp,
\end{align*}
where the last inequality relies on $\sum_{t=1}^{T}\frac{1}{(t+1)^{\beta+1}}\leq \frac{1}{\beta}\lp1-\frac{1}{(T+1)^{\beta}}\rp$.

More specifically, if $\eta_t=\frac{1}{\beta(t+1)}$, then we have
\begin{align*}
	\ls{\mw}^{t+1}-{\mv}^{t+1}\rs
	\leq\lp1+\frac{1}{t+1}\rp\ls{\mw}^{t}-{\mv}^{t}\rs+\frac{4 LC_{\lambda}}{ (t+1)t}+\frac{2L}{(t+1)m\beta}\mathbb{I}_{[j_t=k]}.
\end{align*}
Apply recursive iteration to the above inequality,
\begin{align*}
	\ls{\mw}^{T+1}-{\mv}^{T+1}\rs\leq&2L\sum_{t=1}^{T}\prod_{\tilde{t}=t+1}^{T}\lp1+\frac{1}{\tilde{t}+1}\rp\lp\frac{4 C_{\lambda}}{(t+1)t}+\frac{1}{\beta(t+1)m}\mathbb{I}_{[j_t=k]}\rp\\
	\leq&4L\sum_{t=1}^{T}\exp\lp\log\frac{T+1}{t+1}\rp\lp\frac{4 C_{\lambda}}{(t+1)t}+\frac{1}{\beta(t+1)m}\mathbb{I}_{[j_t=k]}\rp\\
	\leq&4L(T+1)\sum_{t=1}^{T}\frac{1}{t+1}\lp\frac{4 C_{\lambda}}{(t+1)t}+\frac{1}{\beta(t+1)m}\mathbb{I}_{[j_t=k]}\rp\\
	\leq&4L(T+1)\sum_{t=1}^{T}\lp\frac{4C_{\lambda}}{t(t+1)^{2}}+\frac{1}{\beta m(t+1)^{2}}\mathbb{I}_{[j_t=k]}\rp\\
	\leq&4L(T+1)\sum_{t=1}^{T}\lp\frac{4C_{\lambda}}{t^{3}}+\frac{1}{\beta m(t+1)^{2}}\mathbb{I}_{[j_t=k]}\rp\\
	\leq&4L(T+1)\lp2C_{\lambda}+\frac{1}{\beta m}\sum_{t=1}^{T}\frac{1}{ (t+1)^{2}}\mathbb{I}_{[j_{t}=k]}\rp.
\end{align*}
Therefore, we can get
\begin{align*}
	\epsilon_{rk}=4L^2(T+1)\lp2C_{\lambda}+\frac{1}{\beta m}\sum_{t=1}^{T}\frac{1}{ (t+1)^{2}}\mathbb{I}_{[j_{t}=k]}\rp.
\end{align*}
It then follows that 
\begin{align*}
	\frac{1}{mn}\sum_{r=1}^{m}\sum_{k=1}^{n}\epsilon_{rk}^2=\frac{16L^4}{n}(T+1)^{2}\sum_{k=1}^{n}\lp2C_{\lambda}+\frac{1}{\beta m}\sum_{t=1}^{T}\frac{1}{ (t+1)^{2}}\mathbb{I}_{[j_{t}=k]}\rp^2
\end{align*}
and
\begin{align*}
	\frac{1}{mn}\sum_{r=1}^{m}\sum_{k=1}^{n}\epsilon_{rk}=&\frac{4L^2}{mn}(T+1)\sum_{r=1}^{m}\sum_{k=1}^{n}\lp2C_{\lambda}+\frac{1}{\beta m}\sum_{t=1}^{T}\frac{1}{ (t+1)^{2}}\mathbb{I}_{[j_{t}=k]}\rp\\
	=&4L^2(T+1)\lp{2C_{\lambda}}+\frac{1}{\beta mn}\sum_{t=1}^{T}\frac{1}{ (t+1)^{2}}\rp\\
	\leq&4L^2(T+1)\lp{2C_{\lambda}}+\frac{1}{\beta mn}\rp.
\end{align*}
\subsection{Average weight}
	According to the \cref{nonconvex recursion}, we can get
	\begin{equation}
		\ls{\mw}^{t}-{\mv}^{t}\rs\leq4\beta L\sum_{s=1}^{t-1}\eta_{s}\sum_{q=1}^{s-1}\eta_q\lambda^{s-q-1}\prod_{\tilde{t}=s+1}^{t-1}\lp1+\beta\eta_{\tilde{t}}\rp+\frac{2L}{m}\sum_{s=1}^{t-1}\eta_{s}\mathbb{I}_{[j_{s}=k]}\prod_{\tilde{t}=s+1}^{t-1}\lp1+\beta\eta_{\tilde{t}}\rp.
	\end{equation}
	Furthermore, 
	\begin{align*}
		&\ls\bar{\mw}^{T+1}-\bar{\mv}^{T+1}\rs\\
		\leq&\frac{\sum_{t=1}^{T+1}\eta_t\ls\mw^{t}-\mv^{t}\rs}{\sum_{t=1}^{T+1}\eta_t}\\
		\leq&\frac{\sum_{t=1}^{T+1}\eta_t\lp4\beta L\sum_{s=1}^{t-1}\eta_{s}\sum_{q=1}^{s-1}\eta_q\lambda^{s-q-1}\prod_{\tilde{t}=s+1}^{t-1}\lp1+\beta\eta_{\tilde{t}}\rp+\frac{2L}{m}\sum_{s=1}^{t-1}\eta_{s}\mathbb{I}_{[j_{s}=k]}\prod_{\tilde{t}=s+1}^{t-1}\lp1+\beta\eta_{\tilde{t}}\rp\rp}{\sum_{t=1}^{T+1}\eta_t}.
	\end{align*}
	
	If $\eta_t\equiv \eta$,
	\begin{align*}
		\ls{\mw}^{t}-{\mv}^{t}\rs\leq
		\frac{4L\eta}{1-\lambda}\lp1+\beta\eta\rp^{(t-1)}+\frac{2L\eta}{ m}\sum_{s=1}^{t-1}\lp1+\beta\eta\rp^{t-1-s}\mathbb{I}_{[j_{s}=k]}.
	\end{align*}
	Then,
	\begin{align*}
		\ls\bar{\mw}^{T+1}-\bar{\mv}^{T+1}\rs\leq&\frac{\eta\sum_{t=1}^{T+1}\frac{4L\eta}{1-\lambda}\lp1+\beta\eta\rp^{(t-1)}+\frac{2L\eta}{ m}\sum_{s=1}^{t-1}\lp1+\beta\eta\rp^{t-1-s}\mathbb{I}_{[j_{s}=k]}}{(T+1)\eta}\\
		\leq&\frac{4L\lp1+\beta\eta\rp^{(T+1)}}{(1-\lambda)\beta(T+1)}+\frac{2L\eta}{ m(T+1)}\sum_{t=1}^{T+1}\sum_{s=1}^{t-1}\lp1+\beta\eta\rp^{t-1-s}\mathbb{I}_{[j_{s}=k]}\\
		\leq&\frac{4L\lp1+\beta\eta\rp^{(T+1)}}{(1-\lambda)\beta(T+1)}+\frac{2L\eta}{ m(T+1)}\sum_{t=1}^{T+1}\sum_{s=1}^{t-1}\lp1+\beta\eta\rp^{t-1-s}\mathbb{I}_{[j_{s}=k]}.
	\end{align*}
	We can get 
	\begin{align*}
		\epsilon_{rk}=\frac{2L^2}{T+1}\lp\frac{2\lp1+\beta\eta\rp^{(T+1)}}{(1-\lambda)\beta}+\frac{\eta}{ m}\sum_{t=1}^{T+1}\sum_{s=1}^{t-1}\lp1+\beta\eta\rp^{t-1-s}\mathbb{I}_{[j_{s}=k]}\rp.
	\end{align*}
	It then follows that
	\begin{align}\label{ge}
		\frac{1}{mn}\sum_{r=1}^{m}\sum_{k=1}^{n}\epsilon_{rk}^2=\frac{4L^4}{n(T+1)^2}\sum_{k=1}^{n}\lp\frac{2\lp1+\beta\eta\rp^{(T+1)}}{(1-\lambda)\beta}+\frac{\eta}{ m}\sum_{t=1}^{T+1}\sum_{s=1}^{t-1}\lp1+\beta\eta\rp^{t-1-s}\mathbb{I}_{[j_{s}=k]}\rp^2
	\end{align}
	and
	\begin{align*}
		\frac{1}{mn}\sum_{r=1}^{m}\sum_{k=1}^{n}\epsilon_{rk}=&\frac{2L^2}{mn(T+1)}\sum_{r=1}^{m}\sum_{k=1}^{n}\lp\frac{2\lp1+\beta\eta\rp^{(T+1)}}{(1-\lambda)\beta}+\frac{\eta}{ m}\sum_{t=1}^{T+1}\sum_{s=1}^{t-1}\lp1+\beta\eta\rp^{t-1-s}\mathbb{I}_{[j_{t}=k]}\rp\\
		=&\lp\frac{4L^2}{(1-\lambda)\beta}+\frac{2L^2}{ mn\beta^2\eta^2}\rp\frac{\lp1+\beta\eta\rp^{(T+1)}}{T+1}.
	\end{align*}
\section{Optimization Error}\label{Section Optimization}
\subsection{Constant stepsize Case}
\begin{proof}
	According to \cref{Smooth}, since $R_S$ is also $\beta$-smooth, we derive the following:
	\begin{align*}
		&R_S\left(\mw^{t+1}\right)-R_S\left(\mw^t\right)\\
		\leq&\left\langle\mathbf{w}^{t+1}-\mathbf{w}^t, \nabla R_S\lp\mathbf{w}^t\rp\right\rangle+\frac{\beta}{2}\left\|\mathbf{w}^{t+1}-\mathbf{w}^t\right\|_2^2 \\
		\leq&-\frac{\eta_t}{m}\sum_{i=1}^{m}\left\langle\nabla f\lp\mathbf{w}^t(i) ; z_{j_t(i)}\rp, \nabla R_S\left(\mathbf{w}^t\right)\right\rangle
		+\frac{\beta\eta_t^2}{2} \ls\frac{1}{m}\sum_{i=1}^{m}\nabla f\lp\mathbf{w}^t(i) ; z_{j_t(i)}\rp\rs^2 \\
		\leq&\eta_t\left\langle\nabla R_{S}\lp\mw^t\rp-\frac{1}{m}\sum_{i=1}^{m}\nabla f\lp\mw^t ; z_{j_t(i)}\rp,\nabla R_{S}\lp\mw^t\rp\right\rangle\\
		&+\eta_t\left\langle \frac{1}{m}\sum_{i=1}^{m}\nabla f\lp\mw^t ; z_{j_t(i)}\rp-\frac{1}{m}\sum_{i=1}^{m}\nabla f\lp\mw^t(i) ; z_{j_t(i)}\rp,\nabla R_{S}\lp\mw^t\rp\right\rangle\\
		&-\eta_t\ls\nabla R_{S}\lp\mw^t\rp\rs^2+\frac{\beta\eta_t^2}{2} \ls\frac{1}{m}\sum_{i=1}^{m}\nabla f\lp\mathbf{w}^t(i) ; z_{j_t(i)}\rp\rs^2\\
		\overset{\textbf{(a)}}{\leq}&\xi_t+\frac{\eta_t\beta L}{m}\sum_{i=1}^{m}\ls\mw^{t}-\mw^{t}(i)\rs-\eta_t\ls\nabla R_{S}\lp\mw^t\rp\rs^2+\frac{\beta\eta_t^2}{2} \ls\frac{1}{m}\sum_{i=1}^{m}\nabla f\lp\mathbf{w}^t(i) ; z_{j_t(i)}\rp\rs^2\\
		\leq&\xi_t+\frac{\eta_t\beta L}{\sqrt{m}}\lk\sum_{i=1}^{m}\ls{\mw}^{t}-\mw^{t}(i)\rs^2\rk^\frac{1}{2}-\eta_t\ls\nabla R_{S}\lp\mw^t\rp\rs^2+\frac{\beta\eta_t^2}{2} \ls\frac{1}{m}\sum_{i=1}^{m}\nabla f\lp\mathbf{w}^t(i) ; z_{j_t(i)}\rp\rs^2\\
		\leq&\xi_t+2\eta_t\beta L^2\sum_{q=1}^{t-1}\eta_q\lambda^{t-q-1}-\eta_t\ls\nabla R_{S}\lp\mw^t\rp\rs^2+\frac{\beta\eta_t^2}{2} \ls\frac{1}{m}\sum_{i=1}^{m}\nabla f\lp\mathbf{w}^t(i) ; z_{j_t(i)}\rp\rs^2.
	\end{align*}
	The equation \textbf{(a)} leverages smoothness and introduces $\xi_t$, where $\xi_t$ is defined as a martingale difference sequence, given by:
	\begin{equation} 
		\xi_t=\eta_t\left\langle\nabla R_{S}\lp\mw^t\rp-\frac{1}{m}\sum_{i=1}^{m}\nabla f\lp\mw^t ; z_{j_t(i)}\rp,\nabla R_{S}\lp\mw^t\rp\right\rangle.
	\end{equation}
	Next,  we consider the gradient norm and suppose that $j_t=\{j_t(1),j_t(2),\cdots,j_t(m)\}$
	\begin{align*}
		&\left\|\frac{1}{m}\sum_{i=1}^{m}\nabla f\left(\mathbf{w}^t(i) ; z_{j_t(i)}\right)\right\|_2^2\\
		=&\Bigg\|\frac{1}{m}\sum_{i=1}^{m}\nabla f\left(\mathbf{w}^t(i) ; z_{j_t(i)}\right)-\frac{1}{m}\sum_{i=1}^{m}\nabla f\lp\mw^t ; z_{j_t(i)}\rp+\frac{1}{m}\sum_{i=1}^{m}\nabla f\lp\mw^t ; z_{j_t(i)}\rp\\
		&\quad\quad-\nabla R_{S}\lp\mathbf{w}^t\rp+\nabla R_{S}\lp\mathbf{w}^t\rp\Bigg\|_2^2 \\
		\leq&3\ls\frac{1}{m}\sum_{i=1}^{m}\lk\nabla f\left(\mathbf{w}^t(i) ; z_{j_t(i)}\right)-\nabla f\left(\mathbf{w}^t; z_{j_t(i)}\right)\rk\rs^2+3\ls\frac{1}{m}\sum_{i=1}^{m}\nabla f\lp\mw^t ; z_{j_t(i)}\rp-\nabla R_{S}\lp\mathbf{w}^t\rp\rs^2 \\
		&+3\ls\nabla R_{S}\lp\mathbf{w}^t\rp\rs^2\\
		\leq&\frac{3\beta^2}{m}\sum_{i=1}^{m}\ls\mw^{t}-\mw^{t}(i)\rs^2+3\xi_t^{\prime}+3\mathbb{E}_{j_t}\left[\ls\frac{1}{m}\sum_{i=1}^{m}\nabla f\lp\mw^t ; z_{j_t(i)}\rp-\nabla R_{S}\lp\mw^t\rp\rs^2\right]+3\ls\nabla R_{S}\lp\mathbf{w}^t\rp\rs^2\\
		\leq&12\beta^2L^2\lp\sum_{q=1}^{t-1}\eta_q\lambda^{t-q-1}\rp^2+3\xi_t^{\prime}+3\mathbb{E}_{j_t}\left[\ls\frac{1}{m}\sum_{i=1}^{m}\nabla f\lp\mw^t ; z_{j_t(i)}\rp-\nabla R_{S}\lp\mw^t\rp\rs^2\right]+3\ls\nabla R_{S}\lp\mathbf{w}^t\rp\rs^2,
	\end{align*}
	where $\xi_t^{\prime}$ is another martingale difference sequence:
	\begin{align*}
		\xi_t^{\prime}=\ls\frac{1}{m}\sum_{i=1}^{m}\nabla f\lp\mw^t ; z_{j_t(i)}\rp-\nabla R_{S}\lp\mw^t\rp\rs^2-\mathbb{E}_{j_t}\left[\ls\frac{1}{m}\sum_{i=1}^{m}\nabla f\lp\mw^t ; z_{j_t(i)}\rp-\nabla R_{S}\lp\mw^t\rp\rs^2\right].
	\end{align*}
	Thus, $R_S\left(\mathbf{w}^{t+1}\right)$ can be bounded by
	\begin{align}
		&R_S\left(\mw^{t+1}\right)-R_S\left(\mw^t\right)\nonumber\\
		\leq&\xi_t-{\eta_t}\left\|\nabla R_S\left(\mathbf{w}^t\right)\right\|_2^2+\frac{\beta\eta_t^2}{2}\left(3\sigma^2+3\xi_t^{\prime}+3\left\|\nabla R_S\left(\mathbf{w}^t\right)\right\|_2^2+12\beta^2L^2\lp\sum_{q=1}^{t-1}\eta_q\lambda^{t-q-1}\rp^2\right)\nonumber\\
		&\quad+2\eta_t\beta L^2\sum_{q=1}^{t-1}\eta_q\lambda^{t-q-1}\nonumber\\
		\leq&\xi_t-\frac{\eta_t}{2}\left\|\nabla R_S\left(\mathbf{w}^t\right)\right\|_2^2+\frac{3\beta\eta_t^2\sigma^2}{2}+\frac{3\beta\eta_t^2}{2}\xi_t^{\prime}+6\beta^3\eta_t^2L^2\lp\sum_{q=1}^{t-1}\eta_q\lambda^{t-q-1}\rp^2+2\eta_t\beta L^2\sum_{q=1}^{t-1}\eta_q\lambda^{t-q-1}\label{op update},
	\end{align}
	where the last inequality used the assumption $\eta_t \leq \frac{1}{3\beta}$.
	
	Furthermore, 
	\begin{align*}
		R_S\left(\mw^{t+1}\right)=&R_S\left(\mw^{1}\right)+\sum_{s=1}^{t}\lp R_S\left(\mw^{s+1}\right)-R_S\left(\mw^{s}\right)\rp\\
		\leq&\sum_{s=1}^{t}\xi_s-\sum_{s=1}^{t}\frac{\eta_s}{2}\left\|\nabla R_S\left(\mathbf{w}^s\right)\right\|_2^2+\frac{3\beta\sigma^2}{2}\sum_{s=1}^{t}\eta_s^2+\frac{3\beta}{2}\sum_{s=1}^{t}\eta_s^2\xi_s^{\prime}\\
		&+6\beta^3L^2\sum_{s=1}^{t}\eta_s^2\lp\sum_{q=1}^{s-1}\eta_q\lambda^{s-q-1}\rp^2+2\beta L^2\sum_{s=1}^{t}\eta_s\sum_{q=1}^{s-1}\eta_q\lambda^{s-q-1}.
	\end{align*}
	It is easy to verify that $\mathbb{E}_{j_s}\left[\xi_s\right]=0$ and 
	\begin{equation*}
		\left|\xi_s\right| \leq \eta_s\lp\ls\nabla R_{S}\lp\mw^s\rp\rs+\ls\frac{1}{m}\sum_{i=1}^{m}\nabla f\lp\mw^s ; z_{j_s(i)}\rp\rs\rp \ls\nabla R_{S}\lp\mw^s\rp\rs \leq 2\eta_sL^2\leq2\eta L^2.
	\end{equation*}
	Now, let’s examine the inequality for the conditional variances:
	\begin{align*}
		\sum_{s=1}^t \mathbb{E}_{j_s}\left[\left(\xi_s-\mathbb{E}_{j_s}\left[\xi_s\right]\right)^2\right]
		\leq&\sum_{s=1}^t \eta_s^2 \mathbb{E}_{j_s}\left[\left\|\nabla R_{S}\lp\mw^s\rp-\frac{1}{m}\sum_{i=1}^{m}\nabla f\lp\mw^s ; z_{j_s(i)}\rp\right\|_2^2\right]\left\|\nabla R_{S}\lp\mw^s\rp\right\|_2^2 \\
		\leq&\sigma^2 \sum_{s=1}^t \eta_s^2\left\|\nabla R_S\left(\mathbf{w}^s\right)\right\|_2^2\\
		\leq&\frac{\sigma^2}{3 \beta} \sum_{s=1}^t \eta_s\left\|\nabla R_S\left(\mathbf{w}^s\right)\right\|_2^2.
	\end{align*}
	For $\rho=\min \left\{1, \eta\beta L^2 / \sigma^2\right\}$, with probability $1-\delta/2$, we have the following inequality by using Part (ii) of \cref{mag}
	\begin{align*}
		\sum_{s=1}^t \xi_s\leq&\frac{\rho \sigma^2 \sum_{s=1}^t \eta_s\left\|\nabla R_S\left(\mathbf{w}^s\right)\right\|_2^2}{6 \eta\beta L^2}+\frac{2 \eta L^2 \log (2 / \delta)}{\rho} \\
		\leq& \frac{1}{6}\sum_{s=1}^t \eta_s\left\|\nabla R_S\left(\mathbf{w}^s\right)\right\|_2^2+2 \log (2 / \delta) \max \left\{\eta L^2, \sigma^2 / \beta\right\}.
	\end{align*}
	Using the Lipschitz property (See \cref{Lipschitz}), we derive
	\begin{align}
		\xi_s^{\prime} \leq 2\lp\ls\frac{1}{m}\sum_{i=1}^{m}\nabla f\lp\mw^s ; z_{j_s(i)}\rp\rs^2+\ls\nabla R_{S}\lp\mw^s\rp\rs^2\rp \leq 4 L^2.\label{xid}
	\end{align}
	A similar argument shows that $\xi_s^{\prime} \geq-4 L^2$. Thus, applying Part (i) of \cref{mag}, we get with probability at least $1-\delta / 2$:
	\begin{align*}
		\sum_{s=1}^t \eta_s^2 \xi_s^{\prime} \leq 4 L^2\left(2 \sum_{s=1}^t \eta_s^4 \log (2 / \delta)\right)^{\frac{1}{2}} \leq 8 L^2 \log (2 / \delta)+L^2 \sum_{s=1}^t \eta_s^4.
	\end{align*}
	
	Therefore, with probability $1-\delta$, we obtain the following bound for $R_S\left(\mathbf{w}^{t+1}\right)$:
	\begin{align*}
		R_S\left(\mathbf{w}^{t+1}\right) \leq& R_S(\mw^1)+2 \log (2 / \delta) \max \left\{\eta L^2, \sigma^2 / \beta\right\}+\frac{3\beta L^2}{2}\sum_{s=1}^t \eta_s^4 \\
		&-\frac{1}{3}\sum_{s=1}^t \eta_s\left\|\nabla R_S\left(\mathbf{w}^s\right)\right\|_2^2+\frac{3\beta \sigma^2 }{2}\sum_{s=1}^t \eta_s^2+12L^2\beta\log (2 / \delta)\\
		&+6\beta^3L^2\sum_{s=1}^{t}\eta_s^2\lp\sum_{q=1}^{s-1}\eta_q\lambda^{s-q-1}\rp^2+2\beta L^2\sum_{s=1}^{t}\eta_s\sum_{q=1}^{s-1}\eta_q\lambda^{s-q-1}.
	\end{align*}
	Furthermore, according to PL-condition (See \cref{pl}), we know
	\begin{align*}
		&\frac{4\eta\gamma}{3}\sum_{t=1}^{T+1}\lp R_S(\mw^{t})-R_S(\mw_R^{*})\rp
		\leq\frac{\eta}{3}\sum_{t=1}^{T+1} \left\|\nabla R_S\left(\mathbf{w}^t\right)\right\|_2^2\\
		\leq&2 \log (2 / \delta) \max \left\{\eta L^2, \sigma^2 / \beta\right\}+\frac{3\beta L^2}{2}(T+1) \eta^4
		+\frac{3\beta \sigma^2 }{2}(T+1)\eta^2+12L^2\beta\log (2 / \delta)\\
		&+6\beta^3L^2\frac{(T+1)\eta^4}{(1-\lambda)^2}+2\beta L^2\frac{\eta^2(T+1)}{1-\lambda}.
	\end{align*}
	It then follows that
	\begin{align}
		R_S(\bar{\mw}^{T+1})-R_S(\mw_R^{*})=\mathcal{O}\lp\frac{1}{T+1}\log (2 / \delta)+\frac{\eta^3}{(1-\lambda)^2}+\frac{\eta}{(1-\lambda)}\rp\label{OP}.
	\end{align}
	It should be mentioned that (suppose that $\sup_Z f(\mw,Z)=\mathcal{O}\lp\sqrt{n}\rp$ and $Var[f(\mw^*;Z)]\leq\upsilon$)
	\begin{align}
		R_S({\mw^{*}})-R(\mw^*)=\mathcal{O}\lp\frac{\log(1/\delta)}{\sqrt{n}}+\sqrt{\frac{\upsilon\log(1/\delta)}{n}}\rp.\label{Te}
	\end{align}
	By combining \cref{ge,OP,Te}, we can derive the excess risk result in the non-convex setting.
	
\end{proof}
\subsection{Decreasing stepsize case}
\begin{proof}
	We also consider the decreasing stepsize $\eta_t=\frac{2}{\gamma(t+1)}$, and recall the \cref{op update},
	\begin{align*}
		&R_S\left(\mw^{t+1}\right)-R_S\left(\mw^t\right)\\
		\leq&\xi_t-\frac{\eta_t}{2}\left\|\nabla R_S\left(\mathbf{w}^t\right)\right\|_2^2+\frac{3\beta\eta_t^2\sigma^2}{2}+\frac{3\beta\eta_t^2}{2}\xi_t^{\prime}+6\beta^3\eta_t^2L^2\lp\sum_{q=1}^{t-1}\eta_q\lambda^{t-q-1}\rp^2+2\eta_t\beta L^2\sum_{q=1}^{t-1}\eta_q\lambda^{t-q-1}\\
		\leq&\xi_t-\frac{\eta_t}{4}\left\|\nabla R_S\left(\mathbf{w}^t\right)\right\|_2^2+\eta_t\gamma\lk R_S\lp{\mw}_{R}^{*}\rp-R_S(\mw^{t})\rk+\frac{3\beta\eta_t^2\sigma^2}{2}+\frac{3\beta\eta_t^2}{2}\xi_t^{\prime}\\
		&+6\beta^3\eta_t^2L^2\lp\sum_{q=1}^{t-1}\eta_q\lambda^{t-q-1}\rp^2+2\eta_t\beta L^2\sum_{q=1}^{t-1}\eta_q\lambda^{t-q-1}.
	\end{align*}
	Then plug into the stepsize value,
	\begin{align*}
		&\frac{1}{2\gamma(t+1)}\left\|\nabla R_S\left(\mathbf{w}^t\right)\right\|_2^2+R_S\left(\mw^{t+1}\right)-R_S\lp{\mw}_{R}^{*}\rp\\
		\leq&\lp1-\eta_t\gamma\rp\lk R_S(\mw^{t})-R_S\lp{\mw}_{R}^{*}\rp\rk+\xi_t+\frac{3\beta\eta_t^2\sigma^2}{2}\\
		&+\frac{3\beta\eta_t^2}{2}\xi_t^{\prime}
		+6\beta^3\eta_t^2L^2\lp\sum_{q=1}^{t}\eta_q\lambda^{t-q}\rp^2+2\eta_t\beta L^2\sum_{q=1}^{t}\eta_q\lambda^{t-q}\\
		\leq&\frac{t-1}{t+1}\lk R_S(\mw^{t})-R_S\lp{\mw}_{R}^{*}\rp\rk+\xi_t+\frac{6\beta\sigma^2}{\gamma^2(t+1)^2}+\frac{6\beta\sigma^2}{\gamma^2(t+1)^2}\xi_t^{\prime}+\frac{96\beta^3L^2C_{\lambda}^2}{\gamma^4(t+1)^2t^2}+\frac{8\beta L^2C_{\lambda}}{\gamma^2(t+1)t}.
	\end{align*}
	By multiplying both sides by $t(t+1)$, we can get
	\begin{align*}
		&\frac{t}{2\gamma}\left\|\nabla R_S\left(\mathbf{w}^t\right)\right\|_2^2+t(t+1)\lk R_S\left(\mw^{t+1}\right)-R_S\lp{\mw}_{R}^{*}\rp\rk\\
		\leq&t(t-1)\lk R_S\left(\mw^{t}\right)-R_S\lp{\mw}_{R}^{*}\rp\rk+t(t+1)\xi_t+\frac{6\beta\sigma^2}{\gamma^2}+\frac{6\beta\sigma^2}{\gamma^2}\xi_t^{\prime}+\frac{96\beta^3L^2C_{\lambda}^2}{\gamma^4(t+1)t}+\frac{8\beta L^2C_{\lambda}}{\gamma^2}.
	\end{align*}
	Next, summing the above inequality from $t=1$ to $T$, we obtain
	\begin{align*}
		&\frac{1}{2\gamma}\sum_{t=1}^{T}t\left\|\nabla R_S\left(\mathbf{w}^t\right)\right\|_2^2+\sum_{t=1}^{T}t(t+1)\lk R_S\left(\mw^{t+1}\right)-R_S\lp{\mw}_{R}^{*}\rp\rk\\
		\leq&\sum_{t=1}^{T}t(t-1)\lk R_S\left(\mw^{t}\right)-R_S\lp{\mw}_{R}^{*}\rp\rk+\sum_{t=1}^{T}t(t+1)\xi_t+\frac{6\beta\sigma^2}{\gamma^2}\\
		&\quad+\frac{6\beta\sigma^2}{\gamma^2}\sum_{t=1}^{T}\xi_t^{\prime}+\sum_{t=1}^{T}\frac{96\beta^3L^2C_{\lambda}^2}{\gamma^4(t+1)t}+\frac{8\beta L^2C_{\lambda}}{\gamma^2}.
	\end{align*}
	Furthermore,
	\begin{align}
		&\frac{1}{2\gamma}\sum_{t=1}^{T}t\left\|\nabla R_S\left(\mathbf{w}^t\right)\right\|_2^2+T(T+1)\lk R_S\left(\mw^{T+1}\right)-R_S\lp{\mw}_{R}^{*}\rp\rk\nonumber\\
		\leq&\sum_{t=1}^{T}t(t+1)\xi_t+\frac{6\beta\sigma^2T}{\gamma^2}+\frac{6\beta\sigma^2}{\gamma^2}\sum_{t=1}^{T}\xi_t^{\prime}+\sum_{t=1}^{T}\frac{96\beta^3L^2C_{\lambda}^2}{\gamma^4(t+1)t}+\frac{8\beta L^2C_{\lambda}T}{\gamma^2}\label{t+1}.
	\end{align}
	Now we consider the first martingale difference term $t(t+1)\xi_t$,
	\begin{align*}
		|t(t+1)\xi_t|=\left|\frac{2t}{\gamma}\left\langle\nabla R_{S}\lp\mw^t\rp-\frac{1}{m}\sum_{i=1}^{m}\nabla f\lp\mw^t ; z_{j_t(i)}\rp,\nabla R_{S}\lp\mw^t\rp\right\rangle\right|
		\leq\frac{4L^2T}{\gamma}.
	\end{align*}
	We also examine the inequality for the conditional variances:
	\begin{align*}
		&\sum_{t=1}^{T}\mathbb{E}_{j_t}\left[\left(t(t+1)\xi_t-\mathbb{E}_{j_t}\left[t(t+1)\xi_t\right]\right)^2\right]\\
		\leq& \sum_{t=1}^{T}\frac{4t^2}{\gamma^2}\mathbb{E}_{j_t}\left[\left\|\nabla R_{S}\lp\mw^t\rp-\frac{1}{m}\sum_{i=1}^{m}\nabla f\lp\mw^t ; z_{j_t(i)}\rp\right\|_2^2\right]\left\|\nabla R_{S}\lp\mw^t\rp\right\|_2^2 \\
		\leq& \sum_{t=1}^{T}\frac{4\sigma^2t^2}{\gamma^2}\left\|\nabla R_S\left(\mathbf{w}^t\right)\right\|_2^2.
	\end{align*}
	For $\rho=\min \left\{1, L^2 / 4\sigma^2\right\}$, with probability $1-\delta/2$, we have the following inequality 
	\begin{align*}
		\sum_{t=1}^T t(t+1)\xi_t\leq&\frac{\rho \sigma^2 \sum_{t=1}^T t^2\left\|\nabla R_S\left(\mathbf{w}^t\right)\right\|_2^2}{\gamma L^2T}+\frac{4L^2 T\log (2 / \delta)}{\rho\gamma} \\
		\leq& \frac{1}{4\gamma}\sum_{t=1}^T t\left\|\nabla R_S\left(\mathbf{w}^t\right)\right\|_2^2+\frac{4 T}{\gamma}\log (2 / \delta) \max \left\{L^2, 4\sigma^2 \right\}.
	\end{align*}
	From \cref{xid}, we know that $|\xi_t^{\prime}|\leq4L^2$. Therefore, applying Part (i) of \cref{mag}, we derive with probability at least $1-\delta / 2$, $\sum_{t=1}^T \xi_t^{\prime} \leq 4 L^2\left(2T\log (2 / \delta)\right)^{\frac{1}{2}}$. Back to \cref{t+1}, 
	\begin{align*}
		&\frac{1}{2\gamma}\sum_{t=1}^{T}t\left\|\nabla R_S\left(\mathbf{w}^t\right)\right\|_2^2+T(T+1)\lk R_S\left(\mw^{T+1}\right)-R_S\lp{\mw}_{R}^{*}\rp\rk\\
		\leq&\frac{1}{4\gamma}\sum_{t=1}^T t\left\|\nabla R_S\left(\mathbf{w}^t\right)\right\|_2^2+\frac{4 T}{\gamma}\log (2 / \delta) \max \left\{L^2, 4\sigma^2 \right\}+\frac{24\beta L^2\sigma^2}{\gamma^2}\left(2T\log (2 / \delta)\right)^{\frac{1}{2}}\\
		&+\frac{96\beta^3L^2C_{\lambda}^2T}{\gamma^4(T+1)}+\frac{6\beta\sigma^2T}{\gamma^2}+\frac{8\beta L^2C_{\lambda}T}{\gamma^2}.
	\end{align*}
	Then we can get 
	\begin{align*}
		\sum_{t=1}^T t\left\|\nabla R_S\left(\mathbf{w}^t\right)\right\|_2^2\leq&{16 T}\log (2 / \delta) \max \left\{L^2, 4\sigma^2 \right\}+\frac{96\beta L^2\sigma^2}{\gamma}\left(2T\log (2 / \delta)\right)^{\frac{1}{2}}\\
		&+\frac{384\beta^3L^2C_{\lambda}^2T}{\gamma^3(T+1)}+\frac{24\beta\sigma^2T}{\gamma}+\frac{32\beta L^2C_{\lambda}T}{\gamma}.
	\end{align*}
	Furthermore, by the Jensen inequality,  
	\begin{align*}
		R_S({\mw}^{T})-R_S(\mw_R^{*})\leq\left\|\nabla R_S\left(\mathbf{w}^T\right)\right\|_2^2=\mathcal{O}\lp\frac{1}{T}\log (2 / \delta)\rp.
	\end{align*}
\end{proof}

\section{Local Model}\label{Section Local}
\subsection{Concex Case}
\begin{proof}
	According to \cref{local model df}, we know that when the $i$-th worker performs an update, no outliers will be selected from its corresponding sample set.
	If $i\neq r$, we can get
	\begin{align*}
		&\ls{\mw}^{t+1}(i)-{\mv}^{t+1}(i)\rs\\
		=&\ls\sum_{l=1}^{m}P^{t}_{il}{\mw}^{t}(l)-\eta_t\nabla f\lp{\mw}^{t}(i);Z_{j_t(i)}\rp-\sum_{l=1}^{m}P^{t}_{il}{\mv}^{t}(l)+\eta_t\nabla f\lp{\mv}^{t}(i);\tilde{Z}_{j_t(i)}\rp\rs\\
		\leq&\sum_{l\neq i}^{m}P^{t}_{il}\ls{\mw}^{t}(l)-{\mv}^{t}(l)\rs+\ls P^{t}_{ii}{\mw}^{t}(i)-\eta_t\nabla f\lp{\mw}^{t}(i);Z_{j_t(i)}\rp-P^{t}_{ii}{\mv}^{t}(i)+\eta_t\nabla f\lp{\mv}^{t}(i);{Z}_{j_t(i)}\rp\rs\\
		\leq&\sum_{l\neq i}^{m}P^{t}_{il}\ls{\mw}^{t}(l)-{\mv}^{t}(l)\rs+P^{t}_{ii}\ls{\mw}^{t}(i)-{\mv}^{t}(i)\rs\\
		=&\sum_{l=1}^{m}P^{t}_{il}\ls{\mw}^{t}(l)-{\mv}^{t}(l)\rs,
	\end{align*}
	where the second inequality uses the non-expansive operator under $\eta_t\leq\frac{2}{\beta P^{t}_{ii}}$ (Inspired by \cite{bars2023improved}'s Lemma A.3).\\
	If $i=r$, with the probability $1-\frac{1}{n}$, we have
	\begin{align*}
		\ls{\mw}^{t+1}(r)-{\mv}^{t+1}(r)\rs\leq\sum_{l=1}^{m}P^{t}_{rl}\ls{\mw}^{t}(l)-{\mv}^{t}(l)\rs.
	\end{align*}
	If $i=r$, with the probability $\frac{1}{n}$,
	\begin{align*}
		&\ls{\mw}^{t+1}(r)-{\mv}^{t+1}(r)\rs\\
		=&\ls\sum_{l=1}^{m}P^{t}_{rl}{\mw}^{t}(l)-\eta_t\nabla f\lp{\mw}^{t}(r);Z_{k(r)}\rp-\sum_{l=1}^{m}P^{t}_{rl}{\mv}^{t}(l)+\eta_t\nabla f\lp{\mv}^{t}(r);\tilde{Z}_{k(r)}\rp\rs\\
		\leq&\ls\sum_{l=1}^{m}P^{t}_{rl}{\mw}^{t}(l)-\sum_{l=1}^{m}P^{t}_{rl}{\mv}^{t}     (l)\rs+\ls\eta_t\nabla f\lp{\mv}^{t}(r);\tilde{Z}_{k(r)}\rp-\eta_t\nabla f\lp{\mw}^{t}(r);Z_{k(r)}\rp\rs\\
		\leq&\sum_{l=1}^{m}P^{t}_{rl}\ls{\mw}^{t}(l)-{\mv}^{t}(l)\rs+2\eta_tL.
	\end{align*}
	With vector format,
	\begin{align*}
		\begin{bmatrix}
			\ls{\mw}^{t+1}(1)-{\mv}^{t+1}(1)\rs \\ 
			\cdots \\
			\ls{\mw}^{t+1}(r)-{\mv}^{t+1}(r)\rs\\ 
			\cdots\\ 
			\ls{\mw}^{t+1}(m)-{\mv}^{t+1}(m)\rs
		\end{bmatrix}
		\leq P^t
		\begin{bmatrix}
			\ls{\mw}^{t}(1)-{\mv}^{t}(1)\rs \\ 
			\cdots \\
			\ls{\mw}^{t}(r)-{\mv}^{t}(r)\rs\\ 
			\cdots\\ 
			\ls{\mw}^{t}(m)-{\mv}^{t}(m)\rs
		\end{bmatrix}
		+
		\begin{bmatrix}
			0\\ 
			\cdots \\
			2\eta_tL\mathbb{I}_{[j_t=k]}\\ 
			\cdots\\ 
			0
		\end{bmatrix}.
	\end{align*}
	That is
	\begin{align*}
		\delta^{t+1}\leq P^t\delta^{t}+2\eta_tL\mathbb{I}_{[j_t=k]}e_r,
	\end{align*}
	where $e_r$ denotes the standard basis vector.
	
	Following the recursive steps, we derive
	\begin{align}
		\delta^{T+1}\leq 2L\sum_{t=1}^{T}P^{T:t}e_r\eta_t\mathbb{I}_{[j_t(r)=k]},
	\end{align}
	where $P^{T:t}$ represents the product of $T-t+1$ gossip matrices.
	
	Back to $r$-th worker's weight difference,
	\begin{align*}
		\ls{\mw}^{T+1}(r)-{\mv}^{T+1}(r)\rs\leq 2L\sum_{t=1}^{T}e_rP^{T:t}e_r\eta_t\mathbb{I}_{[j_t=k]}
		=2L\sum_{t=1}^{T}P^{T:t}_{rr}\eta_t\mathbb{I}_{[j_t=k]}.
	\end{align*}
	Furthermore, we have 
	\begin{align*}
		\frac{1}{n}\sum_{k=1}^{n}\tilde{\epsilon}_{k}^2=\frac{4L^4}{n}\sum_{k=1}^{n}\lp\sum_{t=1}^{T}P^{T:t}_{rr}\eta_t\mathbb{I}_{[j_t=k]}\rp^2
	\end{align*}
	and 
	\begin{align*}
		\frac{1}{n}\sum_{k=1}^{n}\tilde{\epsilon}_{k}=\frac{2L^2}{n}\sum_{k=1}^{n}\sum_{t=1}^{T}P^{T:t}_{rr}\eta_t\mathbb{I}_{[j_t=k]}
		=\frac{2L^2}{n}\sum_{t=1}^{T}P^{T:t}_{rr}\eta_t.
	\end{align*}
	Then we can get the local generalization error
	\begin{align*}
		\big|R({\mw}^{T+1}(r))-R_{S_r}({\mw}^{T+1}(r))\big|
		\lesssim&\frac{M\log^{\frac{1}{2}}(1/\delta)}{\sqrt{n}}+\lp\frac{1}{n}\sum_{k=1}^{n}\tilde{\epsilon}_{k}^2\rp^{\frac{1}{2}}\log (n)\log (1/\delta)\\
		\lesssim&\frac{M\log^{\frac{1}{2}}(1/\delta)}{\sqrt{n}}+2L\lp\frac{1}{n}\sum_{k=1}^{n}\lp\sum_{t=1}^{T}P^{T:t}_{rr}\eta_t\mathbb{I}_{[j_t=k]}\rp^2\rp^{\frac{1}{2}}\log (n)\log (1/\delta).
	\end{align*}
\end{proof}

\begin{remark}\label{compare 1}
	We also compare with related work \citep{bars2023improved} in the following aspects.
	
	\textbf{1) Differences in Stability Choice:}
	In our work, we use pointwise uniform stability, while \cite{bars2023improved} employ on-average augment stability, which requires taking the expectation over both the data set and the random algorithm in order to derive the generalization error bounds in the expectation sense. This difference in stability choices leads to variations in the type of stability that is analyzed and the corresponding generalization bounds.
	
	\textbf{2) Step Size Selection:}
	Our step size choice is $\eta_t \leq \frac{2}{\beta}$, while theirs is $\eta_t \leq \frac{2 \min_{k} P_{kk}}{\beta}$, where $P_{kk}$ represents the diagonal elements of the gossip matrix, and $P_{kk} \in [0, 1]$. Thus, it is evident that the step size in \cite{bars2023improved}'s work is smaller, which leads to recovering the results of centralized SGD under such conditions. For our pointwise uniform stability results, when the step size is chosen as $\eta = \frac{1}{\sqrt{T}}$, the domain results are also consistent with the stability results for SGD. From this perspective, it can be understood that we recover the results of centralized SGD when considering the stability in the same framework. 
	
	\textbf{3) Impact of Topology Structure:}
	\cite{bars2023improved}'s results suggest that under low-noise conditions (i.e., when the variance bound assumption is zero), the best generalization results are achieved when $C_P = 0$ and $P$ is the identity matrix. It is important to note that $C_P^t = \sum_{s=0}^{t-1} \|P^s - P^{s+1}\|$, and the paper claims that when $C_P = 1$, $P$ becomes a matrix where each entry is $\frac{1}{m}$. However, upon further examination, we found that when $P$ is a matrix with all entries equal to $\frac{1}{m}$, it should satisfy $C_P = 0$. This contradicts their claim that smaller values of $C_P$ bring $P$ closer to the identity matrix (i.e., poorly connected). In other words, the results they present for $C_P = 0$ and $C_P = 1$ are essentially the same value, which contradicts their experimental findings. Furthermore, it is worth noting that the scenario where $P$ is the identity matrix does not provide a fair comparison for decentralized stochastic algorithms, as it reduces to fully local SGD without any communication between nodes. 
\end{remark}

\subsection{Strongly Convex Case}
\begin{theorem}\label{strong convex local}
	Assume that $f(\mw;Z)$ is $\mu$-strongly convex, L-Lipschitz and $\beta$-smooth for any Z with respect to $\mw$. If we run time-varying D-SGD with the stepsize $\eta_t\leq{P^{t}_{rr}}/{\beta}$ for $T$ iterations, we have $\frac{1}{n}\sum_{k=1}^{n}\tilde{\epsilon}_{k}=\frac{2L^2}{n}\sum_{t=1}^{T}\mathcal{P}^{T:t}_{rr}\eta_t$ and 
	\begin{align*}
		\frac{1}{n}\sum_{k=1}^{n}\tilde{\epsilon}_{k}^2=\frac{4L^4}{n}\sum_{k=1}^{n}\lp\sum_{t=1}^{T} \mathcal{P}^{T:t}_{rr}\eta_t\mathbb{I}_{[j_t=k]}\rp^2,
	\end{align*}
	where $\mathcal{P}^t=P^t-\frac{\eta_t\mu}{2}E_m$, and $E_m$ is the identity matrix of size $m\times m$.
\end{theorem}
\begin{proof}
	If $i\neq r$, we can get
	\begin{align*}
		&\ls{\mw}^{t+1}(i)-{\mv}^{t+1}(i)\rs\\
		=&\ls\sum_{l=1}^{m}P^{t}_{il}{\mw}^{t}(l)-\eta_t\nabla f\lp{\mw}^{t}(i);Z_{j_t(i)}\rp-\sum_{l=1}^{m}P^{t}_{il}{\mv}^{t}(l)+\eta_t\nabla f\lp{\mv}^{t}(i);\tilde{Z}_{j_t(i)}\rp\rs\\
		\leq&\sum_{l\neq i}^{m}P^{t}_{il}\ls{\mw}^{t}(l)-{\mv}^{t}(l)\rs+\ls P^{t}_{ii}{\mw}^{t}(i)-\eta_t\nabla f\lp{\mw}^{t}(i);Z_{j_t(i)}\rp-P^{t}_{ii}{\mv}^{t}(i)+\eta_t\nabla f\lp{\mv}^{t}(i);{Z}_{j_t(i)}\rp\rs\\
		\leq&\sum_{l\neq i}^{m}P^{t}_{il}\ls{\mw}^{t}(l)-{\mv}^{t}(l)\rs+P^{t}_{ii}\lp1-\frac{\eta_t\mu}{2P_{ii}^t}\rp\ls{\mw}^{t}(i)-{\mv}^{t}(i)\rs\\
		=&\sum_{l=1}^{m}P^{t}_{il}\ls{\mw}^{t}(l)-{\mv}^{t}(l)\rs-\frac{\eta_t\mu}{2}\ls{\mw}^{t}(i)-{\mv}^{t}(i)\rs.
	\end{align*}
	If $i=r$, we can obtain
	\begin{align*}
		&\ls{\mw}^{t+1}(r)-{\mv}^{t+1}(r)\rs\\
		=&\ls\sum_{l=1}^{m}P^{t}_{rl}{\mw}^{t}(l)-\eta_t\nabla f\lp{\mw}^{t}(r);Z_{j_t(r)}\rp-\sum_{l=1}^{m}P^{t}_{rl}{\mv}^{t}(l)+\eta_t\nabla f\lp{\mv}^{t}(r);\tilde{Z}_{j_t(r)}\rp\rs\\
		\leq&\sum_{l\neq r}^{m}P^{t}_{rl}\ls{\mw}^{t}(l)-{\mv}^{t}(l)\rs+\ls P^{t}_{rr}{\mw}^{t}(r)-\eta_t\nabla f\lp{\mw}^{t}(r);Z_{j_t(r)}\rp-P^{t}_{rr}{\mv}^{t}(r)+\eta_t\nabla f\lp{\mv}^{t}(r);\tilde{Z}_{j_t(r)}\rp\rs\\
		\leq&\sum_{l\neq r}^{m}P^{t}_{rl}\ls{\mw}^{t}(l)-{\mv}^{t}(l)\rs+\ls P^{t}_{rr}{\mw}^{t}(r)-\eta_t\nabla f\lp{\mw}^{t}(r);Z_{j_t(r)}\rp-P^{t}_{rr}{\mv}^{t}(r)+\eta_t\nabla f\lp{\mv}^{t}(r);\tilde{Z}_{j_t(r)}\rp\rs.
	\end{align*}
	Then we derive that
	\begin{align*}
		&\ls{\mw}^{t+1}(r)-{\mv}^{t+1}(r)\rs\\
		\leq&\sum_{l\neq r}^{m}P^{t}_{rl}\ls{\mw}^{t}(l)-{\mv}^{t}(l)\rs+\ls P^{t}_{rr}{\mw}^{t}(r)-\eta_t\nabla f\lp{\mw}^{t}(r);Z_{j_t(r)}\rp-P^{t}_{rr}{\mv}^{t}(r)+\eta_t\nabla f\lp{\mv}^{t}(r);{Z}_{j_t(r)}\rp\rs\\
		&+\ls\eta_t\nabla f\lp{\mv}^{t}(r);\tilde{Z}_{j_t(r)}\rp-\eta_t\nabla f\lp{\mv}^{t}(r);{Z}_{j_t(r)}\rp\rs\\
		\leq&\sum_{l=1}^{m}P^{t}_{rl}\ls{\mw}^{t}(l)-{\mv}^{t}(l)\rs-\frac{\eta_t\mu}{2}\ls{\mw}^{t}(r)-{\mv}^{t}(r)\rs+2\eta_tL\mathbb{I}_{[j_t(r)=k]}.\quad\quad\quad\quad\quad\quad\quad\quad\quad\quad\quad
	\end{align*}
	With vector format,
	\begin{align*}
		\delta^{t+1}\leq \lp P^t-\frac{\eta_t\mu}{2}E_m\rp\delta^{t}+2\eta_tL\mathbb{I}_{[j_t(r)=k]}e_r,
	\end{align*}
	where $E_m$ is the identity matrix.
	
	Following the recursive steps, we obtain
	\begin{align*}
		\delta^{T+1}\leq 2L\sum_{t=1}^{T}\lp P^t-\frac{\eta_t\mu}{2}E_m\rp^{T:t}e_r\eta_t\mathbb{I}_{[j_t(r)=k]},
	\end{align*}
	Back to $r$-th worker's weight difference,
	\begin{align*}
		\ls{\mw}^{T+1}(r)-{\mv}^{T+1}(r)\rs\leq2L\sum_{t=1}^{T}\lp P^t-\frac{\eta_t\mu}{2}E_m\rp^{T:t}_{rr}\eta_t\mathbb{I}_{[j_t(r)=k]}.
	\end{align*}
	Furthermore, we have 
	\begin{align*}
		\frac{1}{n}\sum_{k=1}^{n}\tilde{\epsilon}_{k}^2=\frac{4L^4}{n}\sum_{k=1}^{n}\lp\sum_{t=1}^{T}\lp P^t-\frac{\eta_t\mu}{2}E_m\rp^{T:t}_{rr}\eta_t\mathbb{I}_{[j_t=k]}\rp^2
	\end{align*}
	and 
	\begin{align*}
		\frac{1}{n}\sum_{k=1}^{n}\tilde{\epsilon}_{k}=&\frac{2L^2}{n}\sum_{k=1}^{n}\sum_{t=1}^{T}\lp P^t-\frac{\eta_t\mu}{2}E_m\rp^{T:t}_{rr}\eta_t\mathbb{I}_{[j_t(r)=k]}
		=\frac{2L^2}{n}\sum_{t=1}^{T}\lp P^t-\frac{\eta_t\mu}{2}E_m\rp^{T:t}_{rr}\eta_t.
	\end{align*}
\end{proof}
\subsection{Nonconvex Case}
\begin{theorem}\label{non convex local}
	Suppose that \cref{Lipschitz,Smooth} hold. Then, the pointwise uniform stability of the $r$-th local model in time-varying D-SGD after $T$ iterations is bounded by $\frac{1}{n}\sum_{k=1}^{n}\tilde{\epsilon}_{k}=\frac{2L^2}{n}\sum_{t=1}^{T} \hat{\mathcal{P}}^{T:t}_{rr}\eta_t$ and
	\begin{align*}
		\frac{1}{n}\sum_{k=1}^{n}\tilde{\epsilon}_{k}^2=\frac{4L^4}{n}\lp\sum_{t=1}^{T}\hat{P}^{T:t}_{rr}\eta_t\mathbb{I}_{[j_t=k]}\rp^2,
	\end{align*}
	where $\hat{\mathcal{P}}^t=P^t+{\eta_t\beta}E_{m/r}$, and $E_{m/r}$ is the identity matrix with the $r$-th diagonal element set to 0.
\end{theorem}
\begin{proof}
	If $i\neq r$, we can get
	\begin{align*}
		\ls{\mw}^{t+1}(i)-{\mv}^{t+1}(i)\rs
		=&\ls\sum_{l=1}^{m}P^{t}_{il}{\mw}^{t}(l)-\eta_t\nabla f\lp{\mw}^{t}(i);Z_{j_t(i)}\rp-\sum_{l=1}^{m}P^{t}_{il}{\mv}^{t}(l)+\eta_t\nabla f\lp{\mv}^{t}(i);{Z}_{j_t(i)}\rp\rs\\
		\leq&\sum_{l=1}^{m}P^{t}_{il}\ls{\mw}^{t}(l)-{\mv}^{t}(l)\rs+\eta_t\beta\ls{\mw}^{t}(i)-{\mv}^{t}(i)\rs,
	\end{align*}
	where the last inequality uses the $\beta$-smoothness property.
	
	If $i=r$, we can derive
	\begin{align*}
		\ls{\mw}^{t+1}(r)-{\mv}^{t+1}(r)\rs
		=&\ls\sum_{l=1}^{m}P^{t}_{rl}{\mw}^{t}(l)-\eta_t\nabla f\lp{\mw}^{t}(r);Z_{j_t(r)}\rp-\sum_{l=1}^{m}P^{t}_{rl}{\mv}^{t}(l)+\eta_t\nabla f\lp{\mv}^{t}(r);\tilde{Z}_{j_t(r)}\rp\rs\\
		\leq&\sum_{l=1}^{m}P^{t}_{rl}\ls{\mw}^{t}(l)-{\mv}^{t}(l)\rs+2\eta_tL\mathbb{I}_{[j_t=k]}.
	\end{align*}
	With vector format,
	\begin{align*}
		\delta^{t+1}\leq \lp P^t+{\eta_t\beta}E_{m/r}\rp\delta^{t}+2\eta_tL\mathbb{I}_{[j_t=k]}e_r,
	\end{align*}
	where $E_{m/r}$ is the identity matrix, except for the $r$-th diagonal element, which is 0.
	W
	It is easy to derive that
	\begin{align*}
		\frac{1}{n}\sum_{k=1}^{n}\tilde{\epsilon}_k^2=\frac{4L^4}{n}\sum_{k=1}^{n}\lp\sum_{t=1}^{T}\lp P^t+{\eta_t\beta}E_{m/r}\rp^{T:t}_{rr}\eta_t\mathbb{I}_{[j_t=k]}\rp^2
	\end{align*}
	and 
	\begin{align*}
		\frac{1}{n}\sum_{k=1}^{n}\tilde{\epsilon}_k=\frac{2L^2}{n}\sum_{k=1}^{n}\sum_{t=1}^{T}\lp P^t+{\eta_t\beta}E_{m/r}\rp^{T:t}_{rr}\eta_t\mathbb{I}_{[j_t=k]}
		=\frac{2L^2}{n}\sum_{t=1}^{T}\lp P^t+{\eta_t\beta}E_{m/r}\rp^{T:t}_{rr}\eta_t.
	\end{align*}
\end{proof}
\vskip 0.2in

\bibliography{sample}

\end{document}